\documentclass[10pt]{article} 

\usepackage[preprint]{rlj} 

\usepackage{amssymb} 
\usepackage{mathtools} 
\usepackage{mathrsfs} 
\usepackage{graphicx} 
\usepackage{subcaption} 
\usepackage[space]{grffile} 
\usepackage{url} 
\usepackage{lipsum} 
\usepackage{amsthm}
\usepackage{booktabs}
\usepackage{algorithm}
\usepackage{arydshln}
\usepackage{algpseudocode}
\usepackage{setspace}
\usepackage{multirow}

\theoremstyle{plain}
\newtheorem{theorem}{Theorem}[section]

\newtheorem{lemma}[theorem]{Lemma}
\newtheorem{corollary}[theorem]{Corollary}
\theoremstyle{definition}
\newtheorem{definition}[theorem]{Definition}

\theoremstyle{remark}

\newcommand{\pluseq}{\mathrel{{+}{=}}}

\title{Approximate Next Policy Sampling: Replacing Conservative Target Policy Updates in Deep RL}
\setrunningtitle{Next Policy Sampling}

\author{Dillon Sandhu\textsuperscript{1}, Ronald Parr \textsuperscript{1}}

\emails{dillon.sandhu@duke.edu}

\affiliations{ $^{1}$\textbf{Department of Computer Science, Duke University, USA}\\
\par 
}
\contribution{Proposes Approximate Next Policy Sampling: an alternative to conservative policy updates that explicitly alignins the training distribution with the next policy's state visitation distribution -- rather than constraining the policy update.}
{ As implemented, ANPS allows for unconstrained target policy updates, but must constrain the behavioral policy updates. }

\contribution{ A general bound (Theorem \ref{thm:policy-improvement}) that isolates the importance of the next policy's distribution. This shows why conservative updates are only one possible solution to the distribution mismatch problem raised by \cite{kakade_and_langford:CPI}.}
{ This bound follows straightforwardly from the Performance Difference Lemma \protect\cite{kakade_and_langford:CPI}. }
\contribution{ A bound (Theorem \ref{thm:sv_improvement}) proving that our algorithm guarantees policy improvement by controlling training error and behavioral divergence. }
{ Like standard related bounds (\protect\cite{kakade_and_langford:CPI}, \protect\cite{pmlr-v37-schulman15}), it assumes bounded value error, which is not guaranteed for deep RL. }
\contribution{
A practical wrapper (SV-RL) applicable to standard online RL algorithms (e.g., PPO) that implements ANPS by decoupling the behavioral and target policies. We provide empirical evidence that this allows algorithms to hold the target policy fixed, gather relevant data, and safely execute larger policy jumps. }
{
SV-RL is an initial implementation of ANPS. It relies on a proxy metric (stability of the value estimates) to determine when the value function is accurate enough to update the target policy. Furthermore, it requires off-policy evaluation since it introduces a separate behavioral policy. }

\keywords{Foundations, Policy Improvement} 

\summary{ Conservative policy updates (small changes in policy space) are an integral part of most modern RL algorithms. The origins trace back to a classic "chicken-and-egg" problem in policy improvement: to safely improve a policy, the value function must be accurate on the state-visitation distribution \textit{of the updated policy}, which is unknown during training. To deal with this, conservative methods constrain the policy update -- aiming to ensure this updated state visitation distribution stays similar to the training data.

This paper explores an alternative, which we call \textit{Approximate Next Policy Sampling} (ANPS): shift the training data to be similar to the future policy's state distribution. ANPS aims to improve the value function estimate at the states that are most important to the policy update. In contrast, conservative updates constrain the policy to remain where the value function estimate is already assumed to be trustworthy. We present \textit{Stable Value Approximate Policy Iteration} (SV-API), a lightweight modification to standard approximate policy iteration algorithms that implements ANPS. SV-API holds the target policy fixed while an iteratively updated behavioral policy gathers relevant experience, only committing to a policy update once the value estimates have stabilized. SV-API matches or exceeds the performance of existing methods on high-dimensional discrete (Atari) and continuous control, while making substantially larger target policy updates. This demonstrates the viability of ANPS as an alternative approach to a classic challenge in RL.}
\begin{document}
	\makeCover 
	\maketitle 

    \begin{abstract}
        We revisit a classic "chicken-and-egg" problem in reinforcement learning: to safely improve a policy, the value function must be accurate on the state-visitation distribution \textit{of the updated policy}. That distribution over states is unknown and cannot be sampled for the purposes of training the value function. Conservative updates solve this problem, but at the cost of shrinking the policy update. This paper explores an alternative solution, Approximate Next Policy Sampling (ANPS), which addresses the problem by modifying the training distribution rather than constraining the policy update. ANPS is satisfied if the distribution of the training data approximates that of the next policy. To demonstrate the feasibility and efficacy of ANPS, we introduce Stable Value Approximate Policy Iteration (SV-API). SV-API modifies the standard approximate policy iteration loop to hold the target policy fixed while an iteratively updated behavioral policy gathers relevant experience. It only commits to a new policy once a convergence criterion has been met. If certain stability criteria are met, the update is guaranteed to be safe; otherwise, it remains no less safe than standard approximate policy iteration. Applying SV-API to PPO yields Stable Value PPO (SV-PPO), which matches or improves performance on high-dimensional discrete (Atari) and continuous control benchmarks while executing substantially larger target policy updates. These results demonstrate the viability of ANPS as a new solution to this classic challenge in RL.
    \end{abstract}

\section{Introduction}
	\label{sec:introduction} Policy Iteration \citep{howard:dp} has long been the ``big hammer'' of exact MDP
	algorithms, often solving MDPs which can fit in memory in a shockingly small number of iterations. Combining
	policy iteration with value function approximation, however, undermines the fundamental reason policy iteration
	works: monotonic policy improvement. Various mechanisms for enforcing small policy changes from one iteration to the next have become widely adopted means of mitigating this problem. This paper proposes a novel alternative which, even in its simplest realizations, shows promise.
    
    Value function approximation errors can cause the next policy to be
	worse than the current one \citep{bertsekas1996neuro}. This issue can arise even with a ``well behaved''
	function approximator that minimizes error on its training set -- indeed this behavior is part of the
	problem. States that are rarely visited under the current policy can have high approximation error,
	which can lead to detrimental changes to the policy at these states. The new policy may visit such states
	frequently due to policy changes in other states. The effect is worse performance. Ideally, and
	somewhat counterintuitively, the best distribution for training the current Q-values is the distribution
	of the {\em next} policy, but that typically isn't available when the Q-functions for the current
	policy are learned, creating a chicken-and-egg problem.

	\cite{kakade_and_langford:CPI} introduced {\em conservative policy iteration} as an approach to manage this
	problem. They proved that by limiting the policy change, the distributions for the current and next policy
	would be close, thereby minimizing the problem. Follow up work TRPO \citep{pmlr-v37-schulman15} adapted
	the conservative policy update for actor-critic algorithms, with \cite{kakade_and_langford:CPI}'s result
	playing a central role in the TRPO update derivation. PPO is designed to emulate TRPO using 
    clipping an efficient heuristic \citep{schulman2017proximalpolicyoptimizationalgorithms}. 
    Other modern RL algorithms also regularize changes to the policy to ensure it stays in the region where the Q-values assumed to be more
	trustworthy. (\cite{mpo}, \cite{museli}).

	This paper explores an alternative approach to the distribution mismatch problem: Rather than restricting
	the policy update, our approach actively shifts the data collection to match the next policy. We call it \textit{Approximate
	Next Policy Sampling} (ANPS). Since the value function is trained to minimize error on sampled states,
	ANPS leads to a more trustworthy value estimate (approximately) the states the next policy will visit. If ANPS successfully samples states that the true next policy visits, and the value learning method is able to lower error on these sampled states, then it resolves the chicken-and-egg problem. We propose and analyze an iterative approach to aligning the distributions that can be implemented as a modest change to existing algorithms. We show that idea works as a viable alternative to conservative policy updates - even exceeding the performance of conservative policy updates in some cases.


    \vspace{-1em}
	\subsection{Related Work}
	Standard Approximate Policy Iteration (API) \citep{bertsekas1996neuro} computes the greedy
	policy relative to a $Q$-function estimate for a fixed policy $\pi_k$. Unfortunately, policy improvement cannot be
	guaranteed at each step, meaning the approximately greedy policy $\pi_{k+1}$ may
	perform worse than $\pi_{k}$. As such, the algorithm may not converge to a locally optimum policy. Since the greedy policy update is discontinuous,
	it can lead to a large shift in the state-distribution \citep{perkins_2002}.

	Softer policy update steps are the most widely used solution to this problem. A stochastic policy is
	used, and changed smoothly. For example, policy gradient methods with function approximation \citep{Sutton2000}
	avoid the distribution mismatch problem by using an infinitesimal step in policy space. When the policy change
	is infinitesimal, the change in performance does not depend on the new state visitation distribution, only
	that of the present policy.

	Conservative policy updates (\cite{kakade_and_langford:CPI}, \cite{pmlr-v37-schulman15}) allow for
	discrete, non-infinitesimal jumps. The primary reason for conservative policy updates is mistrust in the value function at the state distribution associated with $\pi_{k+1}$. 

	The Approximate Next Policy Sampling (ANPS) scheme we introduce does not attempt to ensure $\pi'$ has a similar distribution to $\pi$. 
    In contrast, ANPS decouples the target and behavioral policies. 
    Given a prospective next policy $\tilde{\pi}$, ANPS samples it directly in order to estimate performance on $d^{\tilde{\pi}}$.
	This pushes the ``small policy change'' requirement from the target policy to the sampling policy, allowing for larger jumps in policy space, reminiscent of the large jumps of classical Policy Iteration.
    \vspace{-1em}
	\section{Preliminaries}
    
	\label{sec:preliminaries} We consider a Markov Decision Process (MDP) defined by the tuple $(\mathcal{S}
	, \mathcal{A}, P, r, \gamma , s_{0})$, with state space $\mathcal{S}$, action space $\mathcal{A}$,
	transition dynamics $P: \mathcal{S}\times \mathcal{A}\to \Delta(\mathcal{S})$, reward function $r: \mathcal{S}
	\times \mathcal{A}\to [0,1]$, discount factor $\gamma \in [0, 1)$, and fixed initial state $s_{0}$.

	A stationary policy $\pi: S \to \Delta_{A}$ specifies a conditional distribution over actions. We
	overload notation and sometimes use $\pi(s)$ as shorthand for $\pi(\cdot|s)$. The problem is to find a policy
	that maximizes $V^{\pi}(s_{0}) \doteq \mathbb{E}_{\pi}(\sum_{t=0}^{\infty}\gamma^{t}r(s_{t},a_{t})) \ | 
	s_{0})$, where the subscript $\pi$ indicates the expectation is over probabilities that are jointly
	governed by $\pi$ and $P$. The \textbf{action-value function}, $Q^{\pi}$, and \textbf{advantage
	function} are related to the value function as follows, and we refer to all of them as value functions:
	\begin{align}
		Q^{\pi}(s,a) & \doteq r(s,a) + \gamma \mathbb{E}_{s' \sim P(s,a)}\mathbb{E}_{a' \sim \pi(s')}Q^{\pi}(s',a') \label{eq:Q_def} \\
		V^{\pi}(s)   & = \mathbb{E}_{a \sim \pi(s)}Q^{\pi}(s,a)                                                                      \\
		A^{\pi}(s,a) & \doteq Q^{\pi}(s,a) - V^{\pi}(s).\label{eq:A_def}
	\end{align}
	Note that the value function is bounded:
	$\max_{\pi,s,a}Q^{\pi}(s,a) \leq \sum_{i=0}^{\infty}\gamma^{i}= \frac{1}{1-\gamma}$.

	A convenient way to express a policy's performance is in terms of its discounted \textbf{state action
	visitation distribution}. This quantity discounts states based on how far they are in the future, and is
	normalized by $(1-\gamma)$ to ensure it sums to one.
	\begin{align*}
		d^{\pi}(s,a) & \doteq (1-\gamma)\sum_{t=0}^{\infty}\gamma^{t}P(s_{t}= s, a_{t}=a)
	\end{align*}
	The Performance Difference Lemma (PDL; \citet{kakade_and_langford:CPI}) decomposes the improvement of
	arbitrary policy $\pi'$ over arbitrary policy $\pi$ into $A^{\pi}$ and $d^{\pi'}$:
	\begin{align}
		V^{\pi'}(s_{0})-V^{\pi}(s_{0}) = \frac{1}{1-\gamma}\mathbb{E}_{s,a \sim d^{{\pi'}}}\ A^{\pi}(s,a). \tag{PDL} \label{eq:pdl}
	\end{align}

	\section{Approximate Policy Iteration (API)}
	We now describe the distribution shift problem for API, deriving theoretical bounds that emphasize the importance of aligning the training distribution and next-policy's distribution, a concept we term \textit{Next Policy Alignment}. 
	
    Algorithm \eqref{alg:generic_api} presents an on-policy version of API, meaning that it collects data using $\pi_{k}$.  API constructs a sequence of policies $\{\pi_{k}\}_{k=0}^{\infty}$ by iterating between \textit{Data Collection}, \textit{Policy Evaluation} and \textit{Policy Improvement}. 
	In each round $k$, the policy evaluation operator (\texttt{Eval}) estimates  $q_k^{\pi_k}\approx Q^{\pi}$ from dataset $\mathcal{D}_k$.
	Then, the policy improvement operator ($\Gamma$) returns new policy as a function of $q^{\pi_k}_k$, with optional dependence on the present policy $\pi_k$ or data batch $\mathcal{D}_k$. 

	\begin{figure}[ht]
		\centering
		\fbox{%
		\begin{minipage}[t]{0.45\textwidth}
			\begin{algorithm}
				[H]
				\caption{On-Policy API}
				\label{alg:generic_api}
				\begin{algorithmic}
					[1]
                    \setstretch{1.05}
                    \Require API Subroutines $\texttt{Eval}$ and $\Gamma$ \State Initialize $\pi_{0}$ \For{$k = 0, 1, 2, \dots$}
					\State Collect dataset $\mathcal{D}_{k}$ by rolling out $\pi_{k}$. \State $q_{k}^{\pi_k}\leftarrow
					\texttt{Eval}(\pi_{k}, \mathcal{D}_{k})$ \State $\pi_{k+1}\leftarrow \Gamma(q_{k}^{\pi_k}, \pi_{k}, \mathcal{D}_k)$ \EndFor
				\end{algorithmic}
			\end{algorithm}
		\end{minipage}%
		} \hfill \fbox{%
		\begin{minipage}[t]{0.49\textwidth}
			\begin{algorithm}
				[H]
				\caption{Stable Value API (Abstract)}
				\label{alg:sv_api}
				\begin{algorithmic}
					[1] 
                    \setstretch{1.05}
                    \Require \texttt{Eval}, $\Gamma$, and Stability Criterion $\mathcal{C}$
					\State Initialize $\pi_{0}$, $\beta_{0}\gets \pi_{0}$ 
					\For{$k = 0, 1, 2, \dots$} 
						\State Collect dataset $\mathcal{D}_{k}$ by rolling
					out $\beta_{k}$ 
						\State $q^{\pi_k}_{k}\leftarrow \texttt{Eval}(\pi_{k}, \mathcal{D}_{k})$ 
						\State $\beta_{k+1}\leftarrow \Gamma(q^{\pi_k}, \beta_{k}, \mathcal{D}_k)$ 
						\If{$\mathcal{C}(q^{\pi_k}_{k}, q^{\pi_{k-1}}_{k-1}, \mathcal{D}_k)$ is \textbf{True}} 
						\State $\pi_{k+1}\gets \beta_{k+1}$ \Comment{Update target policy}
					\Else
						\State $\pi_{k+1}\gets \pi_{k}$ \Comment{Hold target policy}
					\EndIf
					\EndFor
				\end{algorithmic}
			\end{algorithm}
		\end{minipage}
		}
		\caption{The Stable Value modification to API. SV-API only updates the target policy $\pi$ when the convergence criterion $\mathcal{C}$ is met.}
		\label{fig:alg_comparison}
	\end{figure}

    The importance of the \textbf{next-policy state-action distribution} for policy improvement is well known in the literature on API. We provide an analysis that emphasizes this dependence in each round.
    To quantify the performance of \texttt{Eval} under the training distribution $\mu$, we define the following error metric, which is similar to the Mean Squared Value Error (\cite{sutton2018introduction}).

	\begin{definition}[Weighted Action-Value Error]
        \label{def:q_abs_error}
		The $\mu$-weighted estimation error is
		\begin{align}
			\varepsilon(\mu, q^{\pi}) \doteq \sum_{s\in S}\sum_{a\in A}\mu(s,a) \left| Q^{\pi}(s,a) - q^{\pi}(s,a) \right|. \label{eq:q_abs_error}
		\end{align}
	\end{definition}

	Next, we introduce a quantity that typical policy improvement operators maximize.

	\begin{definition}[Estimated Expected Advantage]
    \label{def:estimated_expected_adv}
		Given critic $q^{\pi}$, policy $\pi$, and arbitrary policy $\pi'$, the \textbf{estimated expected advantage}
		of $\pi'$ over $\pi$ is
		\begin{align}
        \mathbb{A}^{\pi}_{\pi'}\doteq \mathbb{E}_{s,a \sim d^{\pi'}}\left[q^{\pi}(s,a) - V^{\pi}(s)\right].\label{eq:estimated_expected_adv}
		\end{align}
    
	\end{definition}
	Most policy improvement operators approximately or exactly maximize this quantity by selecting actions of (perceived) high value at each state. 
	For example, the greedy policy update exactly maximizes $\mathbb{A}^{\pi}_{\pi'}$. 
	Actor-critic methods maximize a surrogate, which can be estimated on-policy (\cite{pmlr-v37-schulman15}).
    
    
    The following theorem bounds the policy improvement using the quantities in Definition \ref{def:q_abs_error} and Definition \ref{def:estimated_expected_adv}. 
    
	\begin{theorem}[Value Error and Policy Improvement]
		\label{thm:policy-improvement} Given action-value estimate $q^{\pi}$ and target policy $\pi$, let the next policy be $\pi' = \Gamma(q,\pi)$. Then,
		\begin{align}
			\frac{1}{1-\gamma}\left(\mathbb{A}^{\pi}_{\pi'}- \varepsilon(d^{\pi'},q^{\pi}) \right) \leq V^{\pi'}(s_{0}) - V^{\pi}(s_{0}) & \leq \frac{1}{1-\gamma}\left(\mathbb{A}^{\pi}_{\pi'}+ \varepsilon(d^{\pi'},q^{\pi}) \right) \label{eq:thm_1}
		\end{align}
	\end{theorem}
    The proof is in the supplemental materials \ref{sec:proof_thm}.
    Compared to previous results, theorem \ref{thm:policy-improvement} includes the function approximation error of the Q-estimate that is used to improve the policy, making the relationship between this error and performance degradation more salient.
    If $q^\pi$ is completely accurate on $d^{\pi_{k+1}}$, the greedy policy will maximize policy improvement. 
    On the other hand, even if $q^\pi$ is 
    perfectly accurate on-policy, inaccuracy on $d^{\pi_{k+1}}$ means that the greedy update may fail to improve the policy.


    In general, we assume that \texttt{Eval} minimizes $\varepsilon(\mu, q^{\pi})$, by focusing value function capacity on the training data. Thus, $\varepsilon(d^{\pi'}, q^{\pi})$ must be controlled some other way. The following results shows that when the training distribution matches the next round's on-policy distribution, then $\varepsilon(d^{\pi'}, q^{\pi})$ can be controlled indirectly.

	\begin{definition}[\textbf{$\delta$-Next Policy Alignment (NPA)}]
		\label{def:nps} Given next policy $\pi'$, a training distribution $\mu(s,a)$ achieves \textbf{$\delta$-Next
		Policy Alignment} if
		\begin{math}
			d_{TV}(d^{\pi'}, \mu) \leq \delta_{dist} \label{def:npa}
		\end{math}.
	\end{definition}

	\begin{lemma}
		\label{lemma:tv_dist} Assume bounded value error: $\max_{s,a,\pi}|Q^{\pi}(s,a) - q^{\pi}(s,a)| \leq \frac{1}{1-\gamma}$. For any two state-action distributions $\mu$ and $d^{\pi'}$:
		\[
			\varepsilon(d^{\pi'},q^{\pi}) \leq \varepsilon(\mu, q^{\pi}) + \frac{2}{1-\gamma}d_{TV}(d^{\pi'}, \mu
			)
		\]
	\end{lemma} The proof is in Appendix \ref{sec:proof_lemma_tv_dist}. It relies on a basic distribution shift argument relating the expectation over $d^{\pi'}$ present on the LHS to the expectation over $\mu$. 

	\begin{corollary}
    \label{cor:npa_policy_imp}
		Assume the conditions of Theorem \ref{thm:policy-improvement} and Lemma \ref{lemma:tv_dist} hold, and that training distribution $\mu_k$ achieves $\delta$-NPA (Definition \ref{def:npa}). Then, the policy improvement is at least:
		\begin{align}
			V^{\pi'}(s_{0})-V^{\pi}(s_{0}) & \geq \frac{1}{1-\gamma}\mathbb{A}_{\pi'}^\pi - \frac{1}{1-\gamma}\varepsilon(\mu_k, q^\pi) - \frac{2}{(1-\gamma)^2}\delta_{dist} 
		\end{align}
	\end{corollary}
	\begin{proof}
		Substitute the upper bound of $\varepsilon(d^{\pi'}, q^\pi)$ from Lemma 1 directly into the lower bound of Theorem 1.
	\end{proof}
	This shows that when $\delta_{dist}$-NPA occurs, then the policy improvement suffers two penalties: one related to $\delta_{dist}$, and the other related to the training error data. It is similar to bounds due to \cite{kakade_and_langford:CPI} and \cite{pmlr-v37-schulman15}. The former assumes that the greedy policy can be approximately determined only for the on-policy distribution -- naturally constraining the application of the method off-policy. The latter assumes no value function error at all. In contrast, Corollary \ref{cor:npa_policy_imp} explicitly includes the value function error on the training distribution.
    \vspace{-0.5em}
	\subsection{Conservative Policy Iteration (CPI)}
	Conservative policy iteration and related methods ensure that the policy improvement operator in Algorithm \ref{alg:generic_api} performs $\delta$-NPA. In the original CPI \citep{kakade_and_langford:CPI}, the next policy is an exponential moving average (EMA) of the greedy policy with past policies, ensuring $d_{TV}(\pi, \pi')$ is controlled. 

	If $\varepsilon(\mu, q^{\pi})$ is bounded, conservative methods achieve monotonic policy improvement, but at the cost of exceedingly small policy improvements.
	For example, with $\gamma=0.99$, the worst-case policy change, $\max_s d_{TV}(\pi_k(s), \pi_{k+1}(s))$, must be less than one quarter of a percent (computed by applying Theorem 12.2 of \cite{agarwal21}). 
	Theorem \ref{thm:policy-improvement} shows that this can waste data and compute if the value function is already accurate enough on $d^{\pi'}$.

	In practice, deep actor-critic methods implement much weaker conservationism than the sufficient conditions of Corollary \ref{cor:npa_policy_imp} and similar theorems due to \cite{kakade_and_langford:CPI} and \cite{pmlr-v37-schulman15} would suggest.
	For example, in their Museli method, \cite{museli} regularize the update to prefer policies where $d_{TV}(\pi_{new}, \pi_{old}) \leq 0.46$, a jump in policy space that makes the policy improvement bound in Corollary \ref{cor:npa_policy_imp} vacuous. 
	As another example, \cite{Engstrom2020Implementation} found that TRPO and PPO held the average KL divergence at around 0.05 -- still at least $20\times$ too large to guarantee policy improvement.
	These methods are still widely applied, despite this theory-practice gap (\cite{stiennon2020learning}, \cite{OpenAI_GPT4_2023}).

\section{Approximate Next Policy Sampling with Stable Value API}

Our approach, \textbf{Approximate Next Policy Sampling (ANPS)}, achieves NPA through a fundamentally different mechanism. Instead of constraining the next target policy $\pi_{k+1}$ to remain close to $\pi_{k}$ (as in CPI), ANPS decouples data collection from the target policy entirely.

This decoupling allows the algorithm to execute massive, unconstrained updates to the target policy. It achieves this by introducing an iteratively refined behavioral policy that aligns the training distribution with the anticipated next policy before the target update is made.

We introduce Stable Value API (SV-API, Algorithm \ref{alg:sv_api}) as a general framework for executing ANPS. SV-API separates the optimization process into three distinct policies at each iteration $k$:

	\begin{itemize}
		\item \textbf{Target Policy ($\pi_{k}$):} Uniquely determines the value target $Q^{\pi_k}$ for evaluation.
		\item \textbf{Behavioral Policy ($\beta_{k}$):} Collects the data and is iteratively refined to scout the state space.
		\item \textbf{Next Target Policy ($\pi_{k+1}$):} The prospective unconstrained update.
	\end{itemize}

    SV-API gates the target policy update behind a stability condition, $\mathcal{C}$. The behavioral policy gathers data to estimate $Q^{\pi_k}$ and is updated repeatedly. The true value function, $Q^{\pi_k}$, corresponds to a fixed target policy, meaning the behavioral policy can safely explore and align with the optimal actions for that fixed value function while limiting premature target updates.





	The behavioral policy gathers data for the estimation of $Q^{\pi_k}$, and is iteratively updated to help align the behavioral and \textit{next} target policy distributions. 
    The true value function, $Q^{\pi_k}$, corresponds to a policy that may differ from the sampling policy after the first round. 

	Algorithm \ref{alg:sv_api} gates the target policy update behind a stability condition, $\mathcal{C}$, which dictates when the value estimate and distribution have stabilized sufficiently to permit a target policy update. 
	In this work, we explore two concrete instantiations of $\mathcal{C}$:

  \begin{enumerate}
        \item \textbf{Static Intervals:} The target policy is updated every $K$ rounds, where $K$ is a hyperparameter. It is generalized by the following dynamic thresholding scheme.
        
      \item \textbf{Dynamic Thresholding:} The update occurs when the change in the value estimates falls below a hyperparameter $\delta_v$ for at least $K_{min}$ rounds. Specifically, we track the absolute critic change:
      \begin{align}
        \text{diff}_{k} & \doteq \sum_{s,a \in \mathcal{D}}|q_{k}^{\pi_k}(s,a) - q_{k-1}^{\pi_{k-1}}(s,a)|
      \end{align} 
      This quantity is a measurable proxy for $\varepsilon(\mu_{k}, q_{k}^{\pi_k})$ and $d_{TV}(\mu_{k-1},\mu_{k})$. 
      When either of these is large, $\text{diff}_{k}$ may also be large.
      To help ensure transferability across environments, we normalize $\text{diff}_{k}$ by the return before comparing to $\delta_v$.
      The full convergence logic, shown in Algorithm \ref{alg:convergence}, also uses a maximum number of rounds spent evaluating a fixed policy, $K_{max}$ to prevent stalling. 
  \end{enumerate}
Static interval thresholding can be obtained by setting the maximum and minimum number of iterations to the same value: i.e. $K_{min} = K_{max} = K$.
We now present an experiment that allows us to observe the theoretical quantities from Section \ref{sec:preliminaries}, and then return to our theoretical analysis. 


\subsection{An empirical demonstration on the Four Rooms MDP}
\begin{figure}[t]
    \centering
    \fbox{
    \includegraphics[width=.975\textwidth]{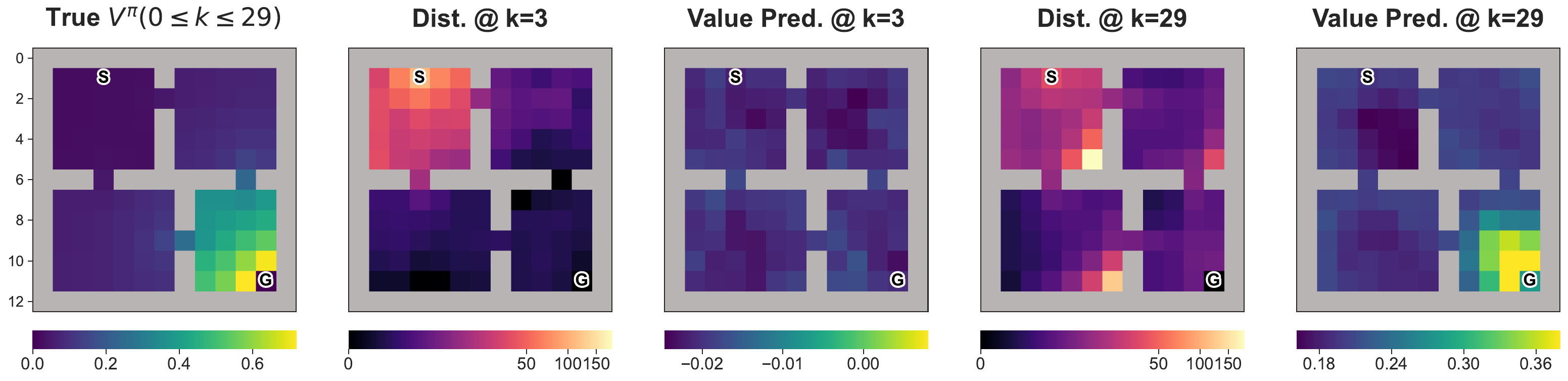}
    }
    \caption{Snapshots two rounds of SV-PPO on Four Rooms. True $V^\pi$ indicates the true value function of the target policy, which is the same in both rounds. ``Dist. @ k'' is the moving average empirical state visitation distribution at round $k$. ``Value Pred.'' is the predicted value.}
    \label{fig:four_rooms_grids}
\end{figure}
To examine for how $\varepsilon(\mu_k, q^{\pi_k})$ and next-policy alignment interact in practice, we examine the Four Rooms environment \citep{journals/ai/SuttonPS99}. We compare a conservative API method (in the form of PPO) baseline against an implementation of our Stable Value API (Algorithm \ref{alg:sv_api}, also applied to PPO). We defer the implementation details to Section \ref{sec:sv_ppo_main}, to focus on learning dynamics.

In standard PPO, data collection is due to the target policy. When applied to Four Rooms, PPO immediately updates the policy based on initial, high-error value estimates. As shown in Figure \ref{fig:four_rooms_curves}, this leads to value churn \citep{tang2024churn} and catastrophic forgetting around iteration 80. The critic overestimates the value of rarely visited regions (the bottom room and leftmost wall) leading to degradation of the policy in these regions. The premature policy updates adjust the policy to avoid these regions, meaning the data required to correct the value estimate is never gathered (See Supplemental Materials \ref{sec:supp_4_rooms} for a detailed explanation).

In contrast, SV-API decouples these elements by holding the target policy fixed while a separate behavioral policy iteratively gathers data. SV-API delays updating the target policy until the change in the value estimates falls below a stability threshold. Figure \ref{fig:four_rooms_grids} shows how this delay allows for stronger policy evaluation of the very first target policy, which is frozen until round $k=30$. During these rounds, the behavioral policy acts as a scout expanding the set of visited states and providing coverage of the goal room. Without this sustained data collection, accurate value estimation of $V^{\pi_0}$ at these distant, high-value states would be impossible.

By evaluating target policies for longer, SV-API suppresses value overestimation, maintains coverage of high-value states, and successfully converges to the optimal policy. Figure \ref{fig:four_rooms_curves}a plots the value of the sampling policy, while Figure \ref{fig:four_rooms_curves}b shows plots of $\varepsilon(\mu_{k+1}, q^{\pi_k}_k)$, the error term in Theorem \ref{thm:policy-improvement}. Value over-estimation and catastrophic forgetting in rarely visited regions lead to spikes in this error for PPO. Figure \ref{fig:four_rooms_curves}c shows that the degree of NPA is similar for both methods, despite the less frequent target policy updates of SV-API. Figure \ref{fig:four_rooms_curves}d shows that the metric tracked for convergence is reasonably correlated with low value error and strong NPA.

    \begin{figure*}[t]
        \centering
        \setlength{\fboxsep}{0pt}
        \setlength{\fboxrule}{1pt}
    
        \fbox{\includegraphics[width=\textwidth]{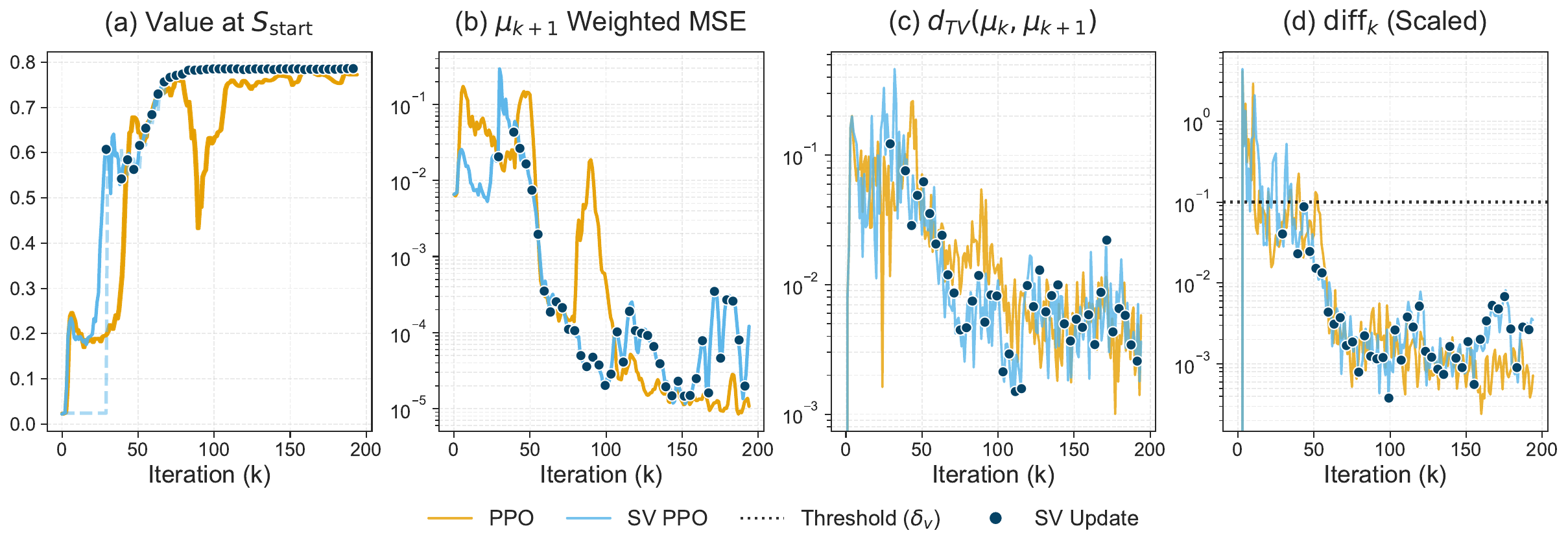}}
        \caption{
        Four Rooms learning curves. Dots indicate target policy updates for SV-API.
        (a) Expected return (value) of the data-collection policy for both algorithms; 
        (b) Next-Policy Weighted Mean Squared Value Error $\mathbb{E}_{\mu_{k+1}}(V
        ^{\pi}(s) - v_{k}(s))^{2}$; 
        (c) TV distance between successive training distributions $\mu_{k}$;
        and (d) Convergence: scaled $\text{diff}_k$ vs. the threshold, $\delta_v$ (dashed line).}
        \label{fig:four_rooms_curves}
    \end{figure*}
    
    \vspace{-0.5em}
	\subsection{Analysis of SV-API}
	\label{sec:analysis_sv_rl}
    To guarantee monotonic policy improvement, SV-API must ultimately achieve $\delta$-NPA in the round that the target policy is updated. SV-API attempts to achieve this by iteratively aligning the behavioral policy, obtained as $\Gamma(q_{k-1}^{\pi_{k-1}}, \beta_{k-1}, \mathcal{D}_{k-1})$, with the next policy, which is obtained as $\Gamma(q_{k}^{\pi_{k}}, \beta_{k}, \mathcal{D}_{k})$.

    Because the target policy update is simply an assignment ($\pi_{k+1} \gets \beta_{k+1}$), the difference between this round's training distribution and the next state distribution is due to the change in the behavioral policy during that final update round. The following Lemma from \cite{agarwal21} formally links a bounded change in policy space to a bounded change in state-action visitation distribution.

	\begin{lemma}
		\label{lemma:tv} For two policies $\beta$ and $\beta'$, if 
		$d_{TV}(\beta'(s),\beta(s)) \leq \delta(1-\gamma)$ for all states $s$, then $d_{TV}(d^{\beta'}(s,a),d^{\beta}
		(s,a)) \leq \delta$.
	\end{lemma}
	\begin{corollary}
		In update round $k$, Algorithm \ref{alg:sv_api} performs $\delta$-NPA if
		$\max_{s}d_{TV}(\beta_{k}(\cdot|s), \beta_{k+1}(\cdot|s)) \leq \delta(1-\gamma)$. \label{cor:anps_svrl}
	\end{corollary}
    This corollary provides the theoretical mechanism for ANPS: the algorithm must iteratively refine the behavioral policy until $\beta_k$ and $\beta_{k+1}$ are sufficiently similar, and then terminate in that round. The next theorem applies the corollary to show that SV-API makes safe, unconstrained, updates to the target policy if the value error is low and the behavioral policy satisfies NPA.

\begin{theorem}[SV-API Improvement]
		\label{thm:sv_improvement} 
		In Algorithm \ref{alg:sv_api}, let the training data distribution $\mu_k$ be the stationary distribution induced by the behavioral policy, $\mu_k \doteq d^{\beta_k}$. Assume the behavioral policy has stabilized such that $d_{TV}(\beta_{k}(s), \beta_{k+1}(s)) \leq \delta(1-\gamma)$ and the target policy update occurs ($\pi_{k+1} \gets \beta_{k+1}$). Finally, assume bounded error on the training distribution $\varepsilon(\mu_{k},q^{\pi_k})\leq \varepsilon$ and globally $\max_{s,a,\pi}|Q^{\pi}(s,a) - q^{\pi}(s,a)| \leq \frac{1}{1-\gamma}$. Then, the performance improvement is bounded by:
		\begin{align}
			V^{\pi_{k+1}}(s_{0}) - V^{\pi_k}(s_{0}) \geq \frac{1}{1-\gamma}\left( \mathbb{A}^{\pi_k}_{\pi_{k+1}} - \varepsilon - \frac{2\delta}{1-\gamma}\right)
		\end{align}
	\end{theorem}
	\begin{proof}
		In an update round, the target policy becomes the new behavioral policy: $\pi_{k+1} = \beta_{k+1}$. 
		The training data distribution $\mu_k$ was generated by rolling out $\beta_k$. 
		Lemma \ref{lemma:tv} says that a maximum policy divergence of $\delta(1-\gamma)$ implies a state-action visitation distribution divergence of at most $\delta$, thus $d_{TV}(\mu_{k}, d^{\pi_{k+1}}) \leq \delta$ (NPA). 
		From here, Corollary \ref{cor:npa_policy_imp} yields the result.
	\end{proof}

	It is instructive to contrast the bound in Theorem \ref{thm:sv_improvement} with the standard monotonic improvement bounds found in conservative methods like CPI \citep{kakade_and_langford:CPI}. In CPI, the updated policy is a mixture of the current policy and the greedy policy. Setting the mixture coefficient to $\delta(1-\gamma)$ restricts the maximum policy divergence to match the conditions of Theorem \ref{thm:sv_improvement}. In exchange, the advantage term is scaled down by $\delta(1-\gamma)$. In other words, standard conservatism inherently penalizes the magnitude of the policy improvement step.\footnote{We compare our Theorem \ref{thm:sv_improvement} to the monotonic policy improvement theorem in more detail in Supplemental Materials \ref{sec:kakade_bound}.}

    In contrast, Theorem \ref{thm:sv_improvement} considers an unconstrained assignment to the target policy. SV-API avoids scaling down the policy improvement because it does not explicitly set $\delta$. Rather, by iteratively applying the improvement operator to the sequence of behavioral policies, the algorithm waits for the behavioral policy to stabilize, ensuring $\delta$-NPA. If this occurs, $\delta$ is kept under control, but the estimated improvement $\mathbb{A}^{\pi_k}_{\pi_{k+1}}$ can still be large because $\pi_{k+1}$ is the accumulation of multiple training steps. In short, a convergent sequence of behavioral policies avoids problematic distribution shift while permitting massive leaps in target policy performance.
    
    This is the motivation behind our proposed dynamic convergence criterion, which simultaneously estimates value function change and behavioral policy change. When the value network stops changing on shifting distribution, this is a proxy for the underlying training distribution stabilizing. This indicates that the next behavioral policy may remain similar, and it is safe to make an update. In this way distribution shift is controlled by the stabilization of the data-collection process, in addition to
    whatever conservative properties $\Gamma$ may have.

	\section{Results}
	\label{sec:sv_ppo_main}
    We modify PPO to implement
	Algorithm \ref{alg:sv_api}. This allows for direct comparison with a strong conservative baseline. Moreover, by preventing large changes in both the value function and (behavioral) policy, PPO's update helps induce the stabilization dynamic described above, where the change in one controls the other. Standard PPO can be obtained from the same code by ensuring $\mathcal{C}$ is always true, for example with $K_{max}=1$.

    SV-PPO (Algorithm \ref{alg:sv_ppo}) freezes a set of policy parameters, which correspond to the target policy. This frozen network is used to compute importance sampling (IS) ratios for the the value target and advantage estimates. These ensure that the estimates track $V^\pi$ and $A^\pi$ despite potentially off-policy data. We use V-Trace \citep{espeholt2018impalascalabledistributeddeeprl} for the value estimation and ReTrace($\lambda$) \citep{DBLP:retrace} for the advantage estimation. 
    
    The clipped IS ratio used in PPO's policy objective corresponds to the ratio of two successive \textit{behavioral} policies.
    This implies the sequence of behavioral polices can drift away from the target policy, but at any round $k$, the difference in the final behavioral policy and next target policy will be controlled in the same fashion as PPO's target policy update. 
    A detailed description all the changes and the loss function are in Supplemental Material \ref{appendix:sv_ppo}.
    
    Compared to PPO, no additional training or environment sampling is necessary.
    In each round, the only additional requirement is running inference of the policy network to obtain the IS weights. Whenever $\delta_{v}= \infty$ and $K_{min}= 1$, it becomes exactly standard PPO. Aggregate results are shown in Table \ref{tab:aggregate_performance}.

    \begin{figure}[t]
		\fbox{
		\begin{minipage}[t]{0.54\textwidth}
			\begin{algorithm}
				[H]
				\caption{Stable Value PPO}
				\label{alg:sv_ppo}
				\begin{algorithmic}
                \setstretch{1.04}
				\Require Target $\pi_{\bar{\theta}}$, Value $v_{\phi_0}$, Num. Env. Steps $T$, 
                \Statex Params: $\delta_{v}, K_{min}, K_{max}, \alpha_{ent}$
                    
					\State $\beta_{{\theta}_{0}}\gets \pi_{\bar{\theta}}$, $k_{\pi}\gets 0$, $n_{\text{stable}}\gets 0$

					\For{$k = 0, 1, 2, \dots$} \State
					$\mathcal{D}_{k}\leftarrow \{(s_{t}, a_{t}, s'_{t}, r_{t})\}_{t=1:T}, a_{t}\sim \beta_{\theta_k}(s_t)$
                    \State \textbf{1. Evaluate Target Policy} (\texttt{Eval})
					\State
					$y_{t}\leftarrow \text{VTrace}(\mathcal{D}_{k}, \beta_{\theta_k}, \pi_{\bar{\theta}})$
					\State
					$\hat{A}_{t}\leftarrow \text{Retrace-GAE}(\mathcal{D}_{k}, \beta_{\theta_k}, \pi_{\bar{\theta}}
					)$
					\State $\phi_{k+1}\leftarrow \text{Step}_{\phi}\left(\frac{1}{2}\sum_{t}(y_{t}- v_{\phi_k}(s_{t}))^{2}\right)$
                    \State \textbf{2. Update Behavioral Policy} ($\Gamma$)
					\State $\theta_{k+1}\leftarrow \text{Step}_{\theta}\left(J^{clip}(\theta_k)+ \alpha_{ent}H(\beta_{\theta_k})\right)$
                    \State \textbf{3. Stability Criterion }  ($\mathcal{C}$)
                    \State $\text{diff}_{k} \gets \frac{1}{T}\sum_{t}|y_{t} - v_{\phi_k}(s_{t})|$
                    \State $\text{stable}, n_{\text{stable}} \gets \texttt{Conv}(y_{1:T}, \text{diff}_{k}, n_{\text{stable}}, k_{\pi})$

					\If{\text{conv.}} 
                    \State $k_{\pi} \gets 0$
                    \State $\bar{\theta} \gets \theta_{k+1}$ \Comment{Update target policy}
					\Else \State $k_{\pi}\pluseq 1$ \Comment{Hold target policy}
					\EndIf \EndFor
				\end{algorithmic}
			\end{algorithm}
		\end{minipage}%
        } 
        \hfill \fbox{%
		\begin{minipage}[t]{0.42\textwidth}
			\begin{algorithm}
				[H]
				\caption{\texttt{Conv}}
				\label{alg:convergence}
				\begin{algorithmic}
					[1] \Require $y_{1:T}$, $\text{diff}_{k}$, $n_{\text{stable}}$, $k_{\pi}$ \Statex In-Scope (Alg.
					\ref{alg:sv_ppo}): $\delta_{v}, K_{min}, K_{max}$

					\State \textbf{Check Stability:} \State $\overline{y}_{T}\gets \frac{1}{T}\sum_{t=1}^{T}|y_{t}|$
					\State $I_{\text{stable}}\gets \mathbb{I}(\text{diff}_{k}\leq \overline{y}_{T}\delta_{v})$
					\Statex \State \textbf{Increment Number of Stable Iters:} \State $n_{\text{stable}}\gets (n_{\text{stable}}
					+ 1) \times I_{\text{stable}}$ \Statex
					
					\State \textbf{Stability Gating / Early Stopping:} \If{$n_{\text{stable}}\geq K_{min}$ \text{or} $k_{\pi}\geq K_{max}$}
					\State \Return \text{True}, $0$ \Else \State \Return \text{False}, $n_{\text{stable}}$ \EndIf
				\end{algorithmic}
			\end{algorithm}
		\end{minipage}
        }
	\end{figure}

    \subsection{Atari}

We run on sixteen Atari Games: the Atari-10 Benchmark (\cite{bellemare13arcade}, \cite{atari_5}), plus six well known and interesting games: Breakout, Asterix, Freeway, Seaquest, Space Invaders, and Ms. Pacman. We run a dynamic variant that sets $\delta=0.01$ and $K_{max}=33$. This configuration was obtained from tuning on MinAtar \cite{young19minatar}. We also run a static interval variant, setting $K=9$. Results are shown in Figure \ref{fig:sv_ppo_atari}.

\begin{figure}[htbp]
        \centering
        \begin{subfigure}[b]{0.48\textwidth}
            \centering
            {\setlength{\fboxsep}{0pt}\setlength{\fboxrule}{0.5pt}\fbox{%
                \includegraphics[width=\textwidth]{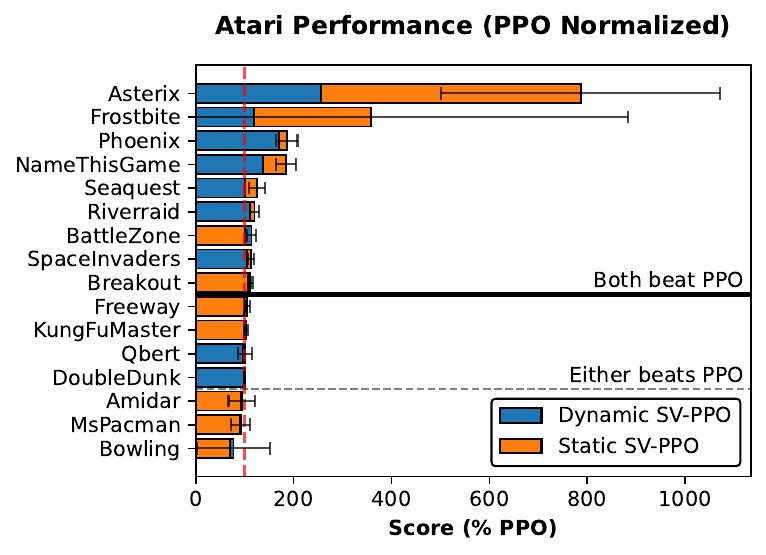}%
            }}
            \caption{Atari benchmark suite. Five Seeds.}
            \label{fig:sv_ppo_atari}
        \end{subfigure}
        \hfill
        \begin{subfigure}[b]{0.48\textwidth}
            \centering
            {\setlength{\fboxsep}{0pt}\setlength{\fboxrule}{0.5pt}\fbox{%
                \includegraphics[width=\textwidth]{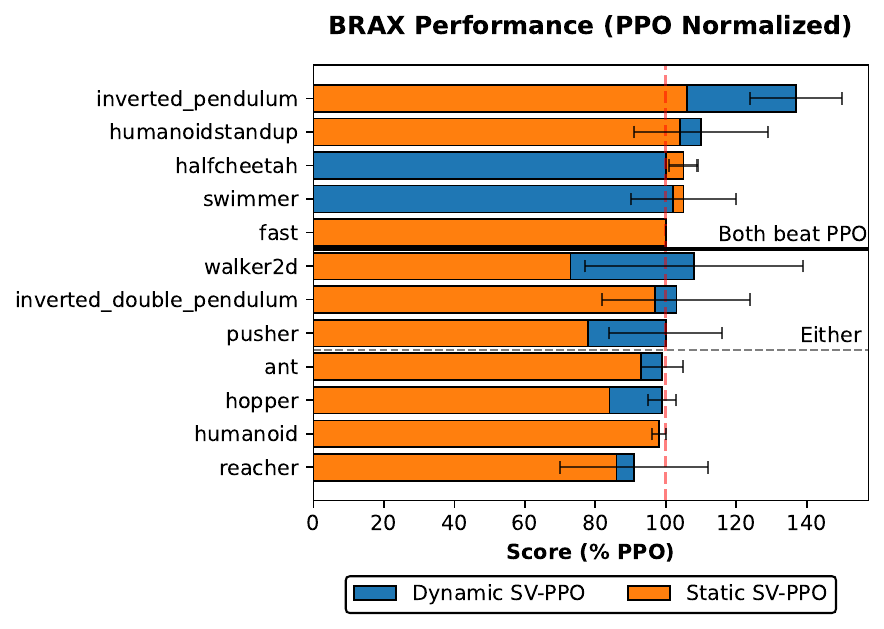}%
            }}
            \caption{BRAX continuous control suite. Twelve seeds.}
            \label{fig:sv_ppo_brax}
        \end{subfigure}
        
        \vspace{2mm} 
        \caption{\textbf{SV-PPO Performance Comparison.} Normalized scores relative to the PPO baseline ($100\%$). Bars show the performance of both methods; the colored tip indicates the margin by which the winning method (Dynamic in blue, Static in orange) outperformed the runner-up.}
        \label{fig:sv_ppo_combined_results}
    \end{figure}

\begin{table}[ht]
\centering
\vspace{2em}
\caption{Aggregate performance of Static and Dynamic SV-PPO across Brax and Atari environments (12 and 5 training seeds respectively). Scores are reported as a percentage of PPO performance. The left column shows median performance, with first and third quartile performance in the bracket. The mean score shows the geometric mean normalized performance, with multiplicative spread shown in parentheses. KL($\pi$) indicates mean change in the target policy during learning, while KL($\beta$) is the same for the behavioral.}
\label{tab:aggregate_performance}
\resizebox{\textwidth}{!}{%
\begin{tabular}{l c c c c c}
\toprule
 & \textbf{Med. Score (\% PPO)} & \textbf{Mean} & \textbf{Beat PPO (\%)} & \textbf{KL ($\pi$)} & \textbf{KL ($\beta$)} \\
\midrule
Static (Brax)   & 97.4 {[}85.7, 104.0{]}   & 93.4 ($\times/\div$ 1.13)   & 42.0 & 14.89$\times$ & 1.40$\times$ \\
Dynamic (Brax)  & 100.2 {[}98.5, 104.2{]}  & 103.2 ($\times/\div$ 1.11)  & 50.0 & 1.08$\times$  & 1.05$\times$ \\
\addlinespace
Static (Atari)  & 104.0 {[}98.8, 139.8{]}  & 133.3 ($\times/\div$ 1.84)  & 56.2 & 6.11$\times$  & 0.94$\times$ \\
Dynamic (Atari) & 105.0 {[}98.2, 115.2{]}  & 113.2 ($\times/\div$ 1.32)  & 68.8& 1.96$\times$  & 0.70$\times$ \\
\bottomrule
\end{tabular}%
}
\end{table}

Performance on all games is shown in Figure \ref{fig:sv_ppo_combined_results}. Dynamic SV-PPO has better mean converged performance on 11 out of the 16 games -- and is statistically significantly better on six ($p=0. 1$, Welch's T-Test), and significantly worse on zero. Static SV-PPO is statistically significantly better on seven. 

Both variants make fewer, larger updates to the target policy. The static interval variant makes exactly one ninth the number of target policy updates, while the dynamic makes between anywhere from 3\% to 95\% as many. The size of the average target policy update is larger by on average 96\% for the dynamic variant, and 511\% for the static variant   (See Table \ref{tab:aggregate_performance}, and Table \ref{tab:sv_atari_results} in the supplemental material for a per-game breakdown). Moreover, the size of the behavioral update is smaller, especially for the dynamic variant. This is empirical verification that SV-PPO can keep the change in sampling policy ($\delta$ in Theorem \ref{thm:sv_improvement}) smaller than PPO's heuristic conservative updates can alone. 

The majority of games benefited from slower target policy updates -- more samples and updates spent learning the value function for a fixed policy. The dynamic variant allows for fast updates at the end (See Figures \ref{fig:scaled_v_diff_all_atari} and \ref{fig:rets_over_time_all_atari} in the supplemental material). In contrast, the static variant slows down the policy updates for all of training. We expected better performance from the dynamic variant, as it helps ensure the policy and value drift is small at the time of the upate. The dynamic convergence criterion resulted in to infrequent policy updates at the beginning (resembling policy iteration) and frequent refinements to the policy at the end (resembling value iteration). However, the static interval variant improved performance on most games where the stable value principle was helpful. This suggests that the convergence heuristic based on could be improved such that these games have even fewer policy updates.
    
    \subsection{Continuous Control}
	We test on continuous control tasks using the Brax physics simulator \citep{brax2021github}. For continuous control, the policy is a Gaussian, and the additional variance due to the off-policy correction appears to be quite high, potentially offsetting the benefits of regularizing value learning. After initially experimenting with $K_{max}=32$, we ended up finding better results setting $K_{max}=8$ to control how off-policy the data could become. The static variant sets $K_{max}=K_{min}=8$, and performs nearly as well as standard PPO with one eighth of the updates.

	The dynamic thresholding variant set $\delta_v = 0.05$ and $K_{max} = 8$. This performed a median of 1\% better across the suite.
  
	SV-PPO can degrade performance on environments where the PPO policy is highly deterministic, such as the \texttt{reacher} and \texttt{ant}. We hypothesize that applying Stable Value API to Soft Actor-Critic (SAC) \cite{conf/icml/HaarnojaZAL18} may help since SAC's entropy regularization could keep the IS ratios closer to one.
    
    \vspace{-0.5em}
    \section{Conclusion}
    This work introduced ANPS, a new mechanism for policy improvement that is based on iteratively refining a behavioral policy until it is suitable to become the new target policy. 
    We implement ANPS using SV-API, which waits until the change in the agent's value and policy predictions falls below a threshold. 
    Our Theorem \ref{thm:sv_improvement} shows that if this heuristic measure is in fact low when the state distribution stops changing, then policy improvement can be guaranteed. Because SV-API can make significantly larger target policy updates, while still ensuring distribution shift does not destroy the policy improvement, we frame it as a new kind of solution to a key problem in policy improvement. 

    To future work we leave (1) Improving the convergence condition $\mathcal{C}$ -- our theory indicates what this should correspond to, but those quantities are unmeasurable; (2) Determining other ways to perform ANPS - for example in imagination with a learned model; (3) Applying techniques from the exploration literature to the SV-API's behavioral policy; and (4) Investigating the degree to which off-policy evaluation harms SV-API in practice, and determining ways to reduce this harm.
    
    \bibliographystyle{rlj}	
    \bibliography{main}
    \clearpage
    \newpage

	\beginSupplementaryMaterials
        \section{Proof of Theorem \ref{thm:policy-improvement}}
    \label{sec:proof_thm}
    	\begin{proof}[Proof of Theorem \ref{thm:policy-improvement}]
		Start with the \ref{eq:pdl}, and introduce $q^{\pi}$ through addition and subtraction:
		\begin{align}
			(1-\gamma)(V^{\pi'}(s_{0})-V^{\pi}(s_{0})) & = \mathbb{E}_{s,a \sim d^{{\pi'}}}\ A^{\pi}(s,a) \quad \eqref{eq:pdl}\nonumber \\
           & = \mathbb{E}_{s,a \sim d^{{\pi'}}}\left[ q^{\pi}(s,a) - V^{\pi}(s) + Q^{\pi}(s,a) - q^{\pi}(s,a) \right] \nonumber \\
           & = \mathbb{A}^{\pi}_{\pi'}+ \mathbb{E}_{s,a \sim d^{{\pi'}}}[ Q^{\pi}(s,a)-q^{\pi}(s,a)] \nonumber
		\end{align}
		Subtract $\mathbb{A}^{\pi}_{\pi'}$ from both sides, and take the absolute value:
		\begin{align}
			\left| (1-\gamma)(V^{\pi'}(s_{0})-V^{\pi}(s_{0})) - \mathbb{A}^{\pi}_{\pi'}\right| & = \left| \mathbb{E}_{s,a \sim d^{{\pi'}}}[Q^{\pi}(s,a)-q(s,a)] \right| \nonumber \\
           & \leq\mathbb{E}_{s,a \sim d^{{\pi'}}}|Q^{\pi}(s,a)-q(s,a)| \nonumber              \\
           & = \varepsilon(d^{\pi'},q^{\pi}) \nonumber
		\end{align}
		Dividing by $(1-\gamma)>0$ and noting $|x| \leq a \iff -a \leq x \leq a$ yields the theorem statement.
        \begin{align}
           &-\varepsilon(d^{\pi'},q^{\pi})  \leq (1-\gamma)(V^{\pi'}(s_{0})-V^{\pi}(s_{0})) - \mathbb{A}^{\pi}_{\pi'} \leq \varepsilon(d^{\pi'},q^{\pi})\nonumber \\
            &\frac{1}{1-\gamma} \left( \mathbb{A}^{\pi}_{\pi'} - \varepsilon(d^{\pi'},q^{\pi}) \right) \leq (V^{\pi'}(s_{0})-V^{\pi}(s_{0}))  \leq \frac{1}{1-\gamma}  \left(\mathbb{A}^{\pi}_{\pi'} + \varepsilon(d^{\pi'},q^{\pi}) \right) \nonumber
        \end{align}
	\end{proof}
    \clearpage
    
	\section{Proof of Lemma \ref{lemma:tv_dist}} 
	\label{sec:proof_lemma_tv_dist}
	Given two distributions, $\mu$, and arbitary action-value estimate $q^\pi$, we'd like to $\varepsilon(\mu' q^\pi)$. 
	Dropping dependence on $q^\pi$, and writing the difference:
	\begin{align}
		\varepsilon(\mu') - \varepsilon(\mu) \nonumber
	\end{align}
	Since $\varepsilon(\mu)$ is positive for any distribution, we have
	\begin{align}
		\varepsilon(\mu') - \varepsilon(\mu) &\leq |\varepsilon(\mu') - \varepsilon(\mu)|\\
		&= \left| \mathbb{E}_{s,a \sim \mu'} |Q^\pi(s,a) - q^\pi(s,a)| - \mathbb{E}_{s,a \sim \mu}|Q^\pi(s,a) - q^\pi(s,a)| \right| \\
		&= \left| \sum_{s,a} (\mu' - \mu) |Q^\pi(s,a) - q^\pi(s,a)| \right|\\
		&\leq \left| \sum_{s,a} |\mu' - \mu| |Q^\pi(s,a) - q^\pi(s,a)| \right|\\
		&= \sum_{s,a} |\mu' - \mu| |Q^\pi(s,a) - q^\pi(s,a)|\\
		&\leq \max_{s,a,\pi} \left(|Q^\pi(s,a) - q^\pi(s,a)|\right) \sum_{s,a} |\mu' - \mu|\\
		&\leq \frac{1}{1-\gamma} \sum_{s,a} |\mu' - \mu| \label{eq:lemma_proof_upper_bound_on_diff} \\
		&= \frac{2}{1-\gamma} d_{TV}(\mu',\mu) 
	\end{align}
	Rearranging gives
	\begin{align}
		\varepsilon(\mu') \leq \varepsilon(\mu) +\frac{2}{1-\gamma} d_{TV}(\mu',\mu) 
	\end{align}
    \clearpage

\section{Comparison of Theorem \ref{thm:sv_improvement} to CPI Monotonic Policy Improvement Bound}
	\label{sec:kakade_bound}

    We compare our Theorem \ref{thm:policy-improvement} to the related monotonic policy improvement bound from CPI, which we state below. Note that the following bound uses $\epsilon$ that is distinct from $\varepsilon(\mu, q^\pi)$. CPI uses $\epsilon$ to denote error in the greedy policy selection subroutine. Recall that in our notation, $\varepsilon(q^{\pi_k},\mu)$ measures the explicit mean value estimation error of the critic on its distribution $\mu$. These terms are related, but are not interchangeable.
	
	\begin{theorem}[\cite{agarwal21} Theorem 12.2]
		Assume access to a method which outputs a policy $\pi'$ such that $\mathbb{A}^{\pi_k}_{\pi'} \geq \max_{\tilde{\pi}} \mathbb{A}^{\pi_k}_{\tilde{\pi}} - \epsilon$. Then, the CPI update $\pi_{k+1} = (1-\delta(1-\gamma)) \pi_k + \delta(1-\gamma) \pi'$ yields the following performance bound:
		\begin{align}
			V^{\pi_{k+1}}(s_{0}) - V^{\pi_k}(s_{0}) \geq \frac{1}{1-\gamma}\left( \delta(1-\gamma) [\max_{\tilde{\pi}} \mathbb{A}^{\pi_k}_{\tilde{\pi}} - \epsilon] - 2\delta^2\gamma \right)
		\end{align}
	\end{theorem}
    Note that, the distribution shift penalty in the last term, is of $O(\delta^2/(1-\gamma))$, while for SV-API it is of $O(\delta/(1-\gamma)^2)$. This is because, for CPI, the coefficient $\delta(1-\gamma)$ scales down both the expected advantage and distribution shift. This allows CPI to guarantee monotonic policy improvement for small enough $\delta$. In contrast, the bound in Theorem \ref{thm:sv_improvement} is due to an unconstrained policy update, and incurs no multiplier of $\delta$ on the improvement $\mathbb{A}_{\pi_{k+1}}^{\pi_k}$.

    \clearpage
	
    \section{Detailed Atari Convergence Threshold}
	\begin{figure}[h]
		\centering
		\includegraphics[width=1\linewidth]{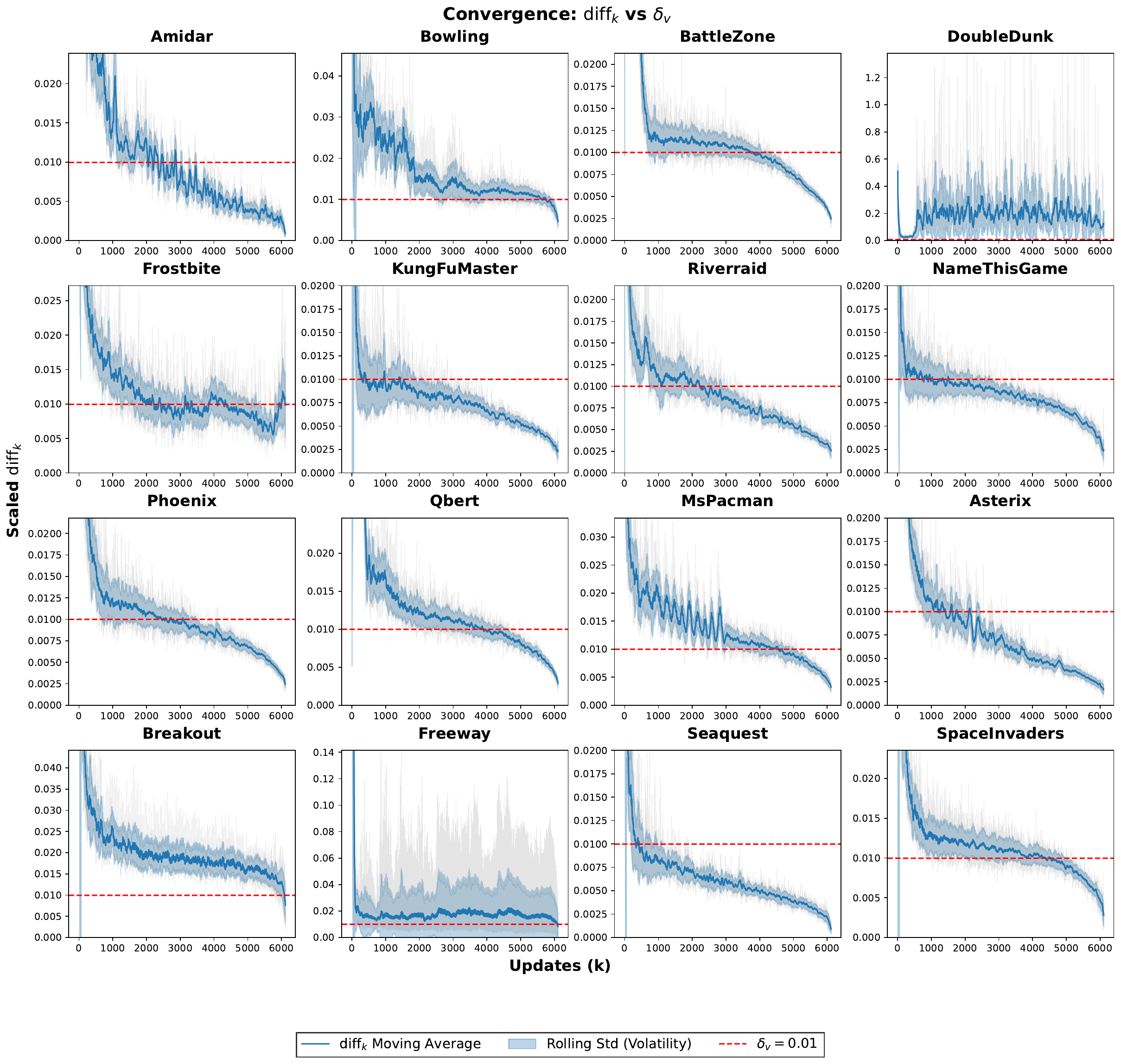}
		\caption{Scaled value diff and convergence threshold $\delta_{v}$ for SV-PPO on all 16 Atari games.}
		\label{fig:scaled_v_diff_all_atari}
	\end{figure}
    \clearpage
    
	\section{Detailed Atari Learning Curves}

	\begin{figure}[h]
		\centering
		\includegraphics[width=1\linewidth]{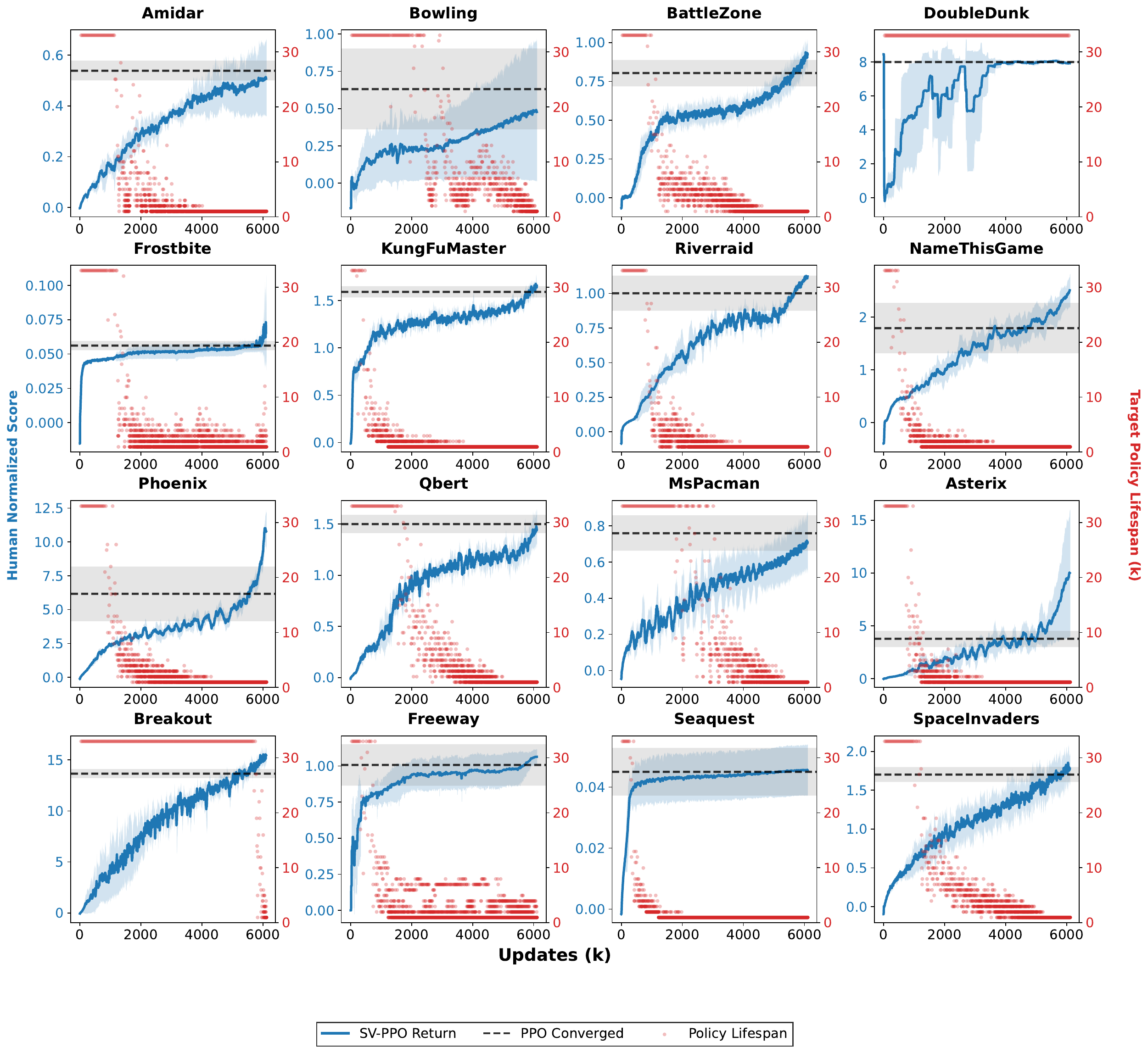}
		\caption{Average score and target policy lifespan for SV-PPO on all 16 Atari games.}
		\label{fig:rets_over_time_all_atari}
	\end{figure}
\clearpage

    \section{Learning Curves: Atari}
        \begin{figure}[htbp]
        \centering
        \includegraphics[width=\textwidth]{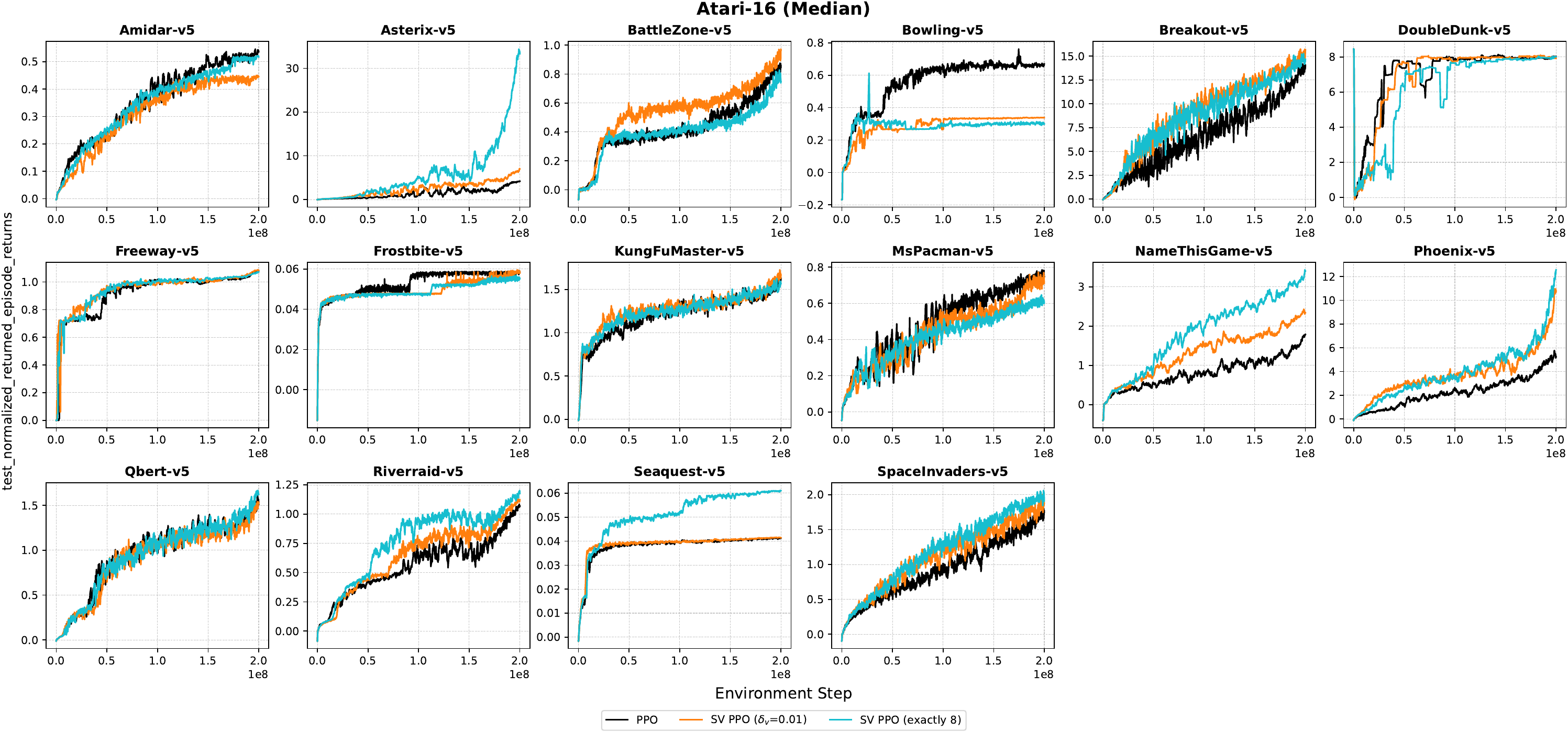}
        \caption{Learning Curves for PPO, Dynamic SV-PPO, and Static SV-PPO for Atari (5 seeds, median shown
        0.}
        \label{fig:sv_ppo_exactly8}
    \end{figure}
    \clearpage
    
	\section{Four Rooms Experiment}  
    \label{sec:supp_4_rooms}
    Four Rooms \citep{journals/ai/SuttonPS99} is a grid world with fixed starting state (top left room) and goal state (bottom right room)
    The action is the intended direction of movement, and is executed noisily, with a $20\%$ chance of success. 
    Hugging the walls mitigates the negative effect of an erroneous move.
    The optimal policy, shown in Figure (\ref{fig:fourrooms_optimal}), visits all four rooms: generally going through the top room, but occasionally going through the bottom room if brought there via noise.
    
	We use 10,000 rounds of value iteration to solve for the optimal policy on Four Rooms. We then compute the optimal actions, placing equal weight on ties, and compute the associated discounted stationary distribution.
	\label{sec:appendix_four_rooms}
	\begin{figure}[htbp]
		\centering
		\begin{subfigure}
			[b]{0.43\textwidth}
			\centering
			\includegraphics[width=\textwidth]{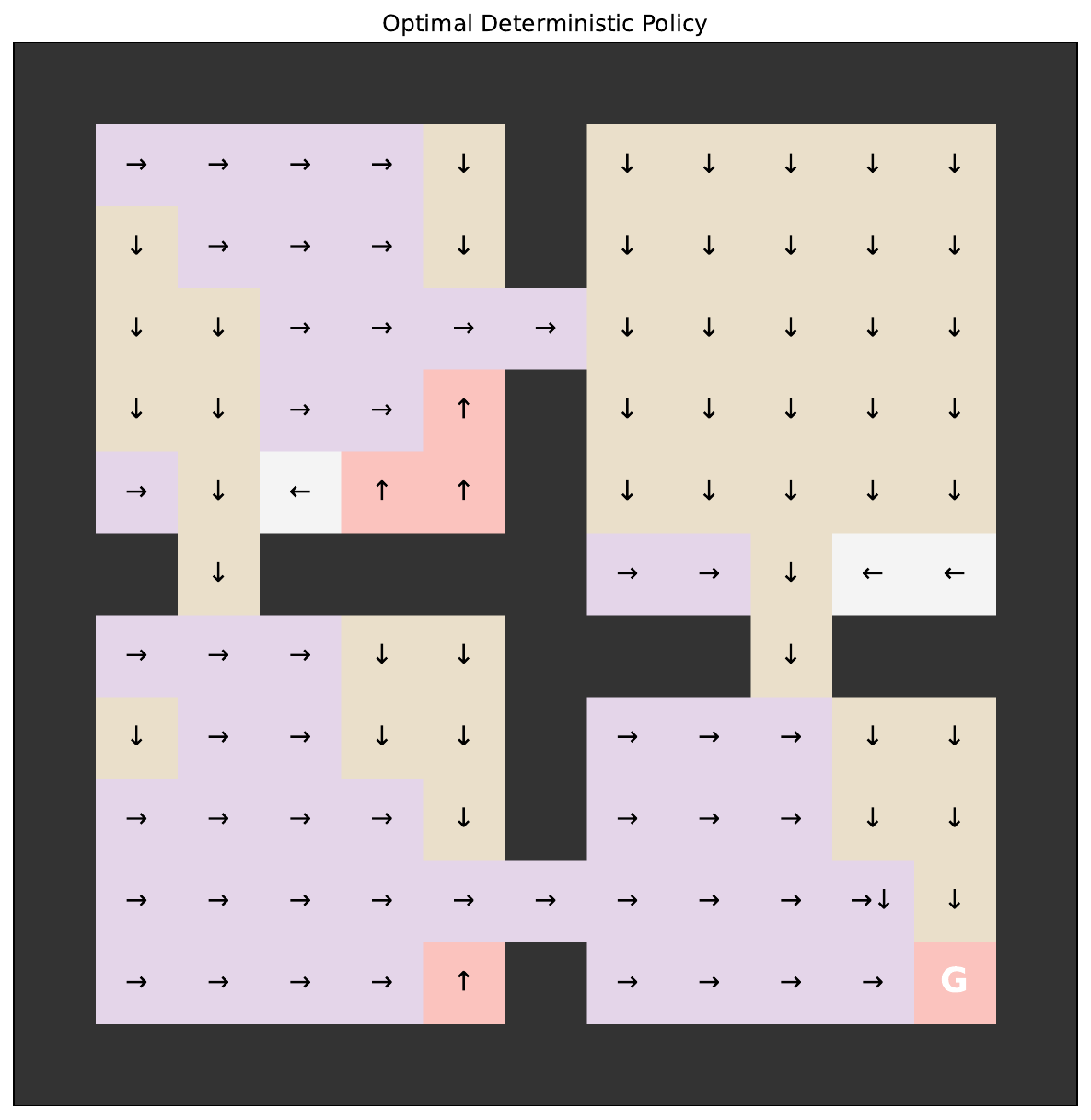}
			\caption{Optimal Policy}
			\label{fig:fourrooms_optimal}
		\end{subfigure}
		\hfill 
		\begin{subfigure}
			[b]{0.51\textwidth}
			\centering
			\includegraphics[width=\textwidth]{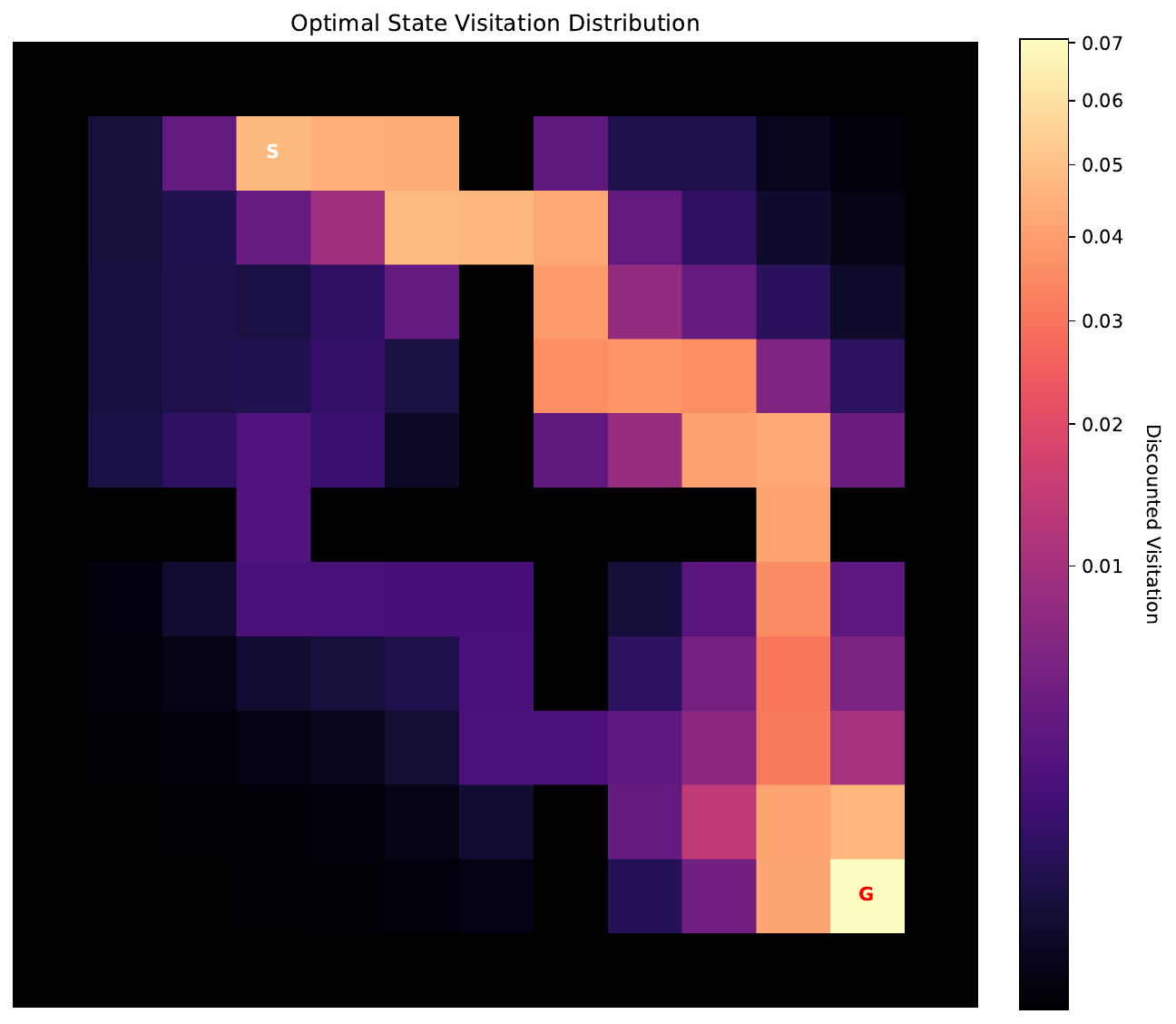}
			\caption{Associated State Visitation Distribution}
			\label{fig:fourroomsOptVisit}
		\end{subfigure}
		\caption{Optimal Policies and Associated State Visitation Distribution}
		\label{fig:four_rooms_optimal_total}
	\end{figure}

    \subsection{Experiment Details}
	We run an implementation of PPO due to \cite{purejaxrl}, using a convolutional actor critic network, due to \cite{pqn}). 
    The observation space is visual (the entire grid, with the agent's position and wall marked).
    SV-PPO uses dynamic thresholding, and is discussed in detail in the next section. 
    We run for 195 total rounds, each corresponding to an update to the behavioral policy. 
    Based on the strictness of the convergence criterion, SV-PPO updates the target policy 40 times.
    Figure \ref{fig:four_rooms_curves} visualizes the convergence logic (corresponding to $\mathcal{C}$ in Algorithm \ref{alg:sv_api}): If the quantity shown in blue on the rightmost chart falls below the dotted line for four consecutive iterations, an update occurs.
    \subsection{Results}
    
	In terms of performance, SV-PPO converges to the optimal policy.
	PPO has unstable convergence to a local optimum that always attempts to go through the top rooms even if noise has pushed the agent closer to the bottom-left room, while the optimal policy is to continue along the bottom route.
	Of note is the sudden collapse in the performance of PPO in iteration $k=80$. 
	This is due to mis-generalization of the value network, sometimes called value churn \citep{tang2024churn}. 
	Around iteration $k =50$, PPO's CNN critic over-estimates the value along the left wall.
	By $k=60$, the PPO policy network learns to avoid the bottom room. 
	Insufficient sampling and poor generalization leads to catastrophic forgetting occurring at iteration $k=90$.
    This is indicated by a massive degradation in the true value function at these states.
    The ``recovery'' in value that occurs for PPO is actually to avoid these states, which is a sub-optimal policy.
    Detailed visualizations of the learning dynamics are in Figures \eqref{fig:ppo_detailed} and \eqref{fig:sv_ppo_detailed}. 

	By evaluating a fixed policy for longer, SV-PPO mitigated value overestimation.
    For example, it experienced less-overestimation along the left wall, since it was fitting a lower value policy in the top right room.
	Lower value estimates led to more consistent exploration of the problematic room -- instead of just seeking the left wall, as PPO did.
    The end result is that SV-PPO learns a strong policy in the bottom left room, while PPO's remains weak, even at the end of learning.

	\begin{figure}[htbp]
		\centering
		\includegraphics[width=\linewidth]{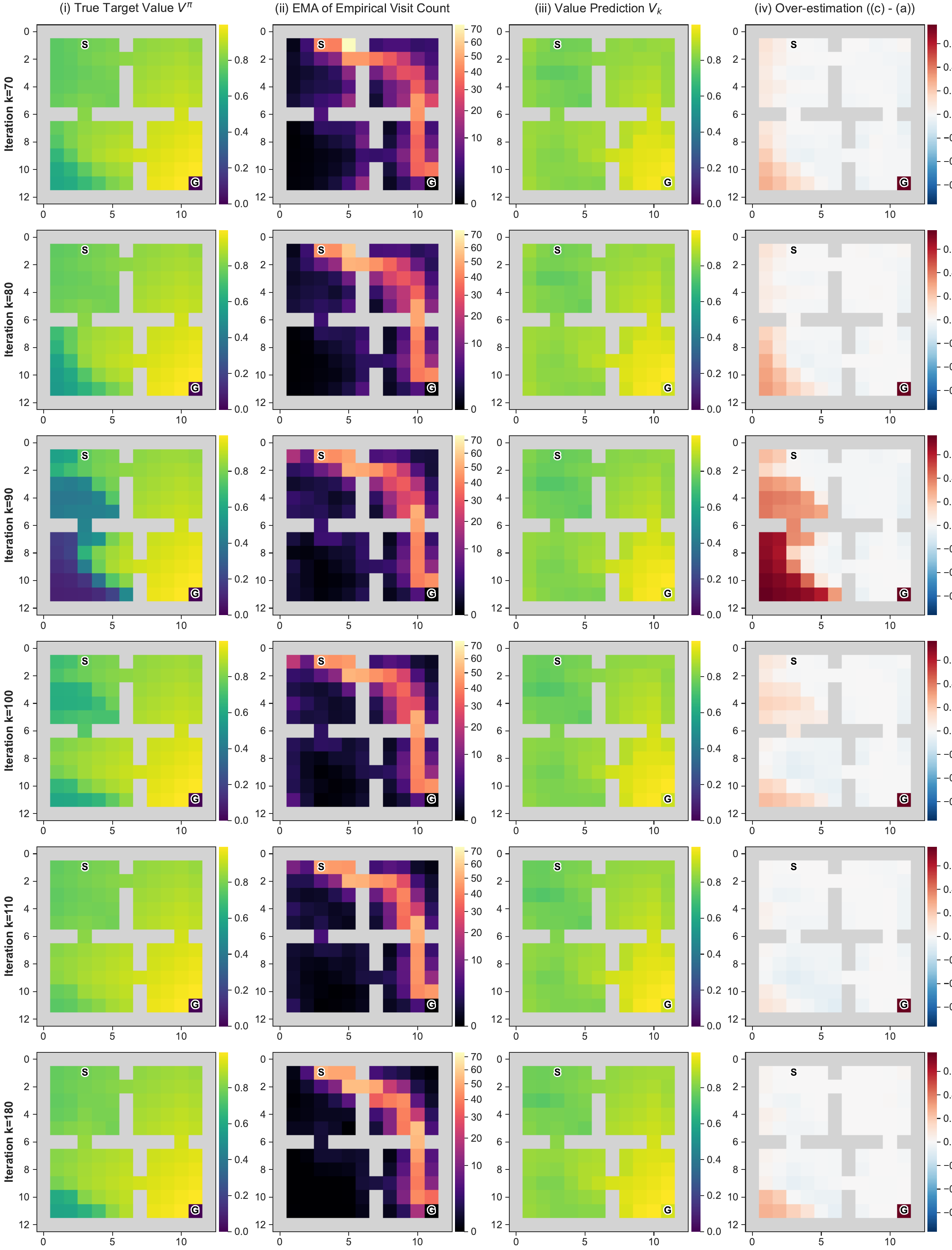}
		\caption{Detailed Examination of PPO's Learning on Four Rooms, highlighting the catastrophic forgetting
		at iteration 90, and over exploration of the bottom left corner due to value over -estimation at Iteration
		100. This leads the agent to avoid the room altogether.}
		\label{fig:ppo_detailed}
	\end{figure}

	\begin{figure}[htbp]
		\centering
		\includegraphics[width=\linewidth]{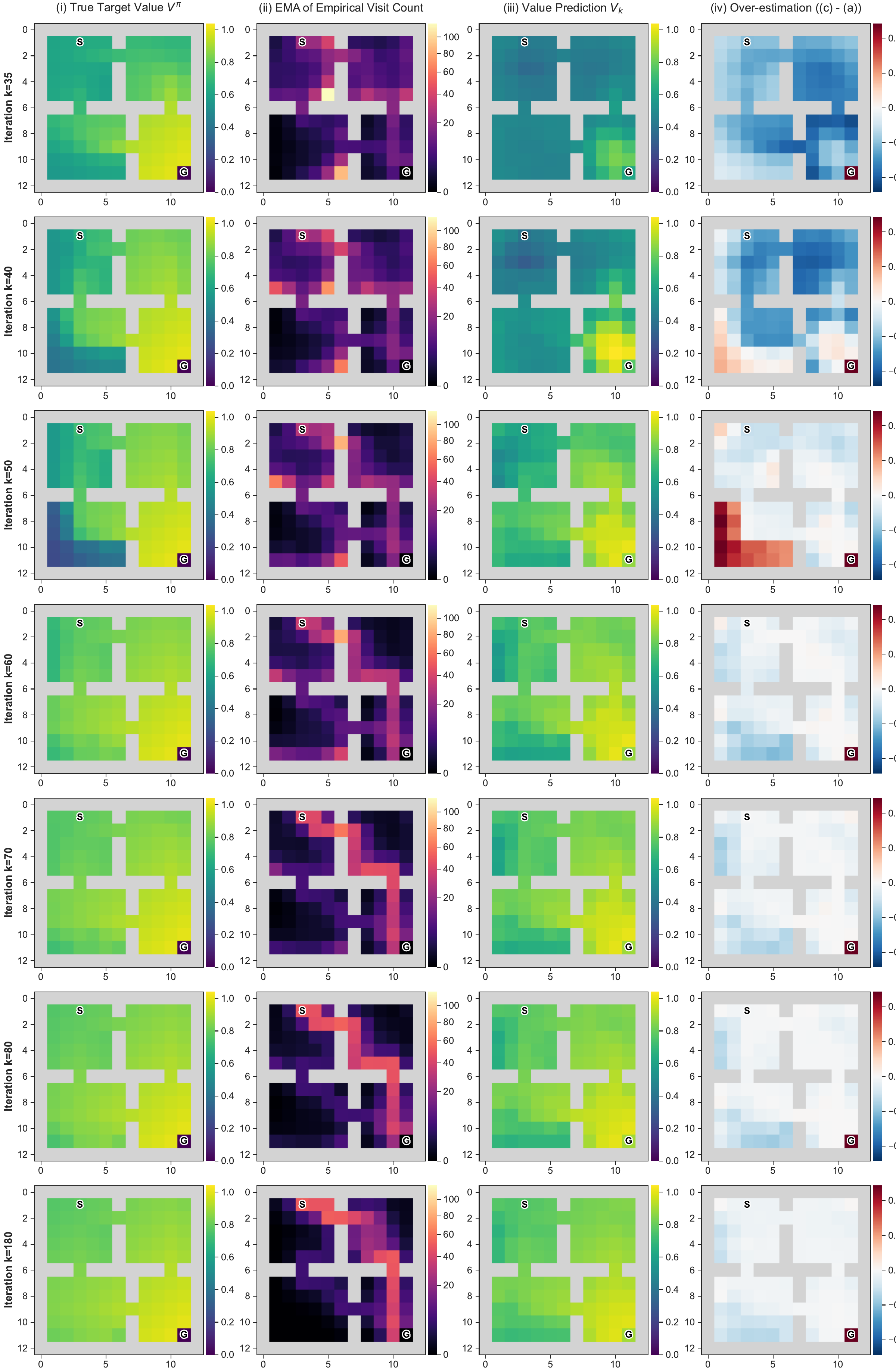}
		\caption{Detailed Examination of SV-PPO's Learning on Four Rooms, highlighting iterations which saw value
		over-estimation the bottom left corner, but less widespread than PPO, and more quickly resolved by exploring
		the room. The final policy goes through the bottom room just like the optimal policy in Figure (\ref{fig:four_rooms_optimal_total})}
		\label{fig:sv_ppo_detailed}
	\end{figure}
	\clearpage

    \section{Static SV-PPO Atari}
    	\begin{table}[htbp]
		\caption{\textbf{Static SV-PPO on Atari.} By strictly gathering data for 9 rounds before every target policy update, SV-PPO operates on 11\% of the target updates of standard PPO as shown in \textbf{(b)
		Updates}. 
        \textbf{Column (a)} indicates the final converged return relative to PPO. \textbf{Column (b)} indicates the fraction of rounds $k$ where SV-PPO updates the target policy. 
        \textbf{Column (c)} shows the size of the target policy jump as a multiplier of PPO's, demonstrating that SV-PPO executes substantially larger target updates. 
        \textbf{Column (d)} shows the same quantity for the behavioral policy remains 
        constrained.}
		\label{tab:sv_atari_results_8}
		\begin{center}
			\resizebox{\textwidth}{!}{
			\begin{tabular}{lcccc}
				\toprule Environment             & \shortstack{(a)\\Score\\(\% of PPO)} & \shortstack{(b)\\Updates\\(\% of PPO)} & \shortstack{(c)\\KL($\pi_{k}|| \pi_{k+1}$)\\($\times$ PPO)} & \shortstack{(d)\\KL($\beta_{k}|| \beta_{k+1}$)\\($\times$ PPO)} \\
				\midrule Asterix                 & \textbf{787 $\pm$ 285}                        & 11                                     & 4.5x                                                        & 0.7x                                                          \\
				Frostbite                        & 358 $\pm$ 526                        & 11                                     & 7.6x                                                        & 1.1x                                                          \\
				Phoenix                          & \textbf{186 $\pm$ 22}                         & 11                                     & 4.5x                                                        & 0.8x                                                          \\
				NameThisGame                     & \textbf{184 $\pm$ 20}                         & 11                                     & 5.3x                                                        & 1.0x                                                          \\
				Seaquest                         & \textbf{125 $\pm$ 17}                         & 11                                     & 6.9x                                                        & 1.1x                                                          \\
				Riverraid                        & \textbf{120 $\pm$ 9}                          & 11                                     & 7.3x                                                        & 1.0x                                                          \\
				SpaceInvaders                    & \textbf{113 $\pm$ 6}                          & 11                                     & 4.7x                                                        & 0.9x                                                          \\
				Breakout                         & \textbf{107 $\pm$ 3}                          & 11                                     & 6.2x                                                        & 0.9x                                                          \\
				Qbert                            & 101 $\pm$ 14                         & 11                                     & 7.5x                                                        & 0.9x                                                          \\
				DoubleDunk                       & 100 $\pm$ 1                          & 11                                     & 11.3x                                                       & 2.0x                                                          \\
				BattleZone                       & 100 $\pm$ 11                         & 11                                     & 6.6x                                                        & 0.9x                                                          \\
				\hdashline \noalign{\vskip 0.5ex} 
				Freeway                          & 99 $\pm$ 14                          & 11                                     & 4.9x                                                        & 1.0x                                                          \\
				KungFuMaster                     & 98 $\pm$ 8                           & 11                                     & 6.1x                                                        & 0.9x                                                          \\
				Amidar                           & 92 $\pm$ 10                          & 11                                     & 5.7x                                                        & 0.9x                                                          \\
				MsPacman                         & 90 $\pm$ 14                          & 11                                     & 6.7x                                                        & 0.9x                                                          \\
				Bowling                          & 70 $\pm$ 32                          & 11                                     & 4.9x                                                        & 0.6x                                                          \\
				\bottomrule
			\end{tabular}
			}
		\end{center}
	\end{table}
    \clearpage
    
    \section{Dynamic SV-PPO Atari}

	\begin{table}[htbp]
		\caption{\textbf{Dynamic SV-PPO on Atari 2600 environments.} \textbf{Column (a)} indicates the final
		converged return relative to PPO, with bold indicating significantly different performance. \textbf{Column
		(b)} indicates the fraction of rounds $k$ where SV-PPO updates the target policy. \textbf{Columns (c)
		and (d)} show statistics about the convergence value ($\text{diff}_{k}/\overline{y}_{T}$ in Algorithm
		\ref{alg:convergence}), normlaized by the threshold $\delta_{v}$, such that a value $\leq 1$ implies stability.\textbf{Column
		(e)} shows the size of the target policy jump as a multiplier of PPO's, demonstrating that SV-PPO executes
		substantially larger target updates. \textbf{Column (f)} shows the same quantity for the behavioral policy.}
		\label{tab:sv_atari_results}
		\begin{center}
			\resizebox{\textwidth}{!}{
			\begin{tabular}{lcccccc}
				\toprule Environment             & \shortstack{(a)\\ \textbf{Score}\\ (\% of PPO)} & \shortstack{(b)\\ \textbf{Updates}\\(\% of PPO)} & \shortstack{(c)\\ {\boldmath \textbf{$\text{diff}_{k}/\delta_{v}$}} \\(Median)} & \shortstack{(d)\\ {\boldmath \textbf{Std. $\text{diff}_{k}$}} \\(\% of $\delta_{v}$)} & \shortstack{(e)\\ {\boldmath \textbf{KL($\pi_{k}|| \pi_{k+1}$)}} \\(x PPO)} & \shortstack{(f)\\ {\boldmath \textbf{KL($\beta_{k}|| \beta_{k+1}$)}} \\(x PPO)} \\
				\midrule Asterix                 & \textbf{256 $\pm$ 151}                          & 74                                               & 56                                                                              & 27                                                                                    & 0.7x                                                                        & 0.5x                                                                          \\
				Phoenix                          & \textbf{170 $\pm$ 16}                           & 61                                               & 85                                                                              & 25                                                                                    & 0.9x                                                                        & 0.5x                                                                          \\
				NameThisGame                     & \textbf{137 $\pm$ 16}                           & 75                                               & 75                                                                              & 21                                                                                    & 1.3x                                                                        & 1.1x                                                                          \\
				Frostbite                        & 119 $\pm$ 40                                    & 55                                               & 83                                                                              & 37                                                                                    & 1.3x                                                                        & 0.6x                                                                          \\
				BattleZone                       & \textbf{114 $\pm$ 10}                           & 47                                               & 96                                                                              & 27                                                                                    & 1.8x                                                                        & 0.8x                                                                          \\
				Riverraid                        & \textbf{111 $\pm$ 1}                            & 67                                               & 73                                                                              & 27                                                                                    & 1.2x                                                                        & 0.7x                                                                          \\
				Breakout                         & \textbf{111 $\pm$ 5}                            & 4                                                & 187                                                                             & 90                                                                                    & 9.6x                                                                        & 0.7x                                                                          \\
				Freeway                          & 105 $\pm$ 5                                     & 38                                               & 106                                                                             & 189                                                                                   & 2.3x                                                                        & 0.9x                                                                          \\
				SpaceInvaders                    & 105 $\pm$ 12                                    & 41                                               & 100                                                                             & 24                                                                                    & 1.4x                                                                        & 0.6x                                                                          \\
				KungFuMaster                     & 103 $\pm$ 3                                     & 79                                               & 67                                                                              & 21                                                                                    & 1.2x                                                                        & 1.0x                                                                          \\
				Seaquest                         & 101 $\pm$ 18                                    & 85                                               & 47                                                                              & 18                                                                                    & 1.2x                                                                        & 1.0x                                                                          \\
				\hdashline \noalign{\vskip 0.5ex} 
				DoubleDunk                       & 99 $\pm$ 1                                      & 3                                                & 1645                                                                            & 3299                                                                                  & 27.0x                                                                       & 0.9x                                                                          \\
				Qbert                            & 96 $\pm$ 11                                     & 40                                               & 100                                                                             & 28                                                                                    & 1.9x                                                                        & 0.6x                                                                          \\
				Amidar                           & 94 $\pm$ 27                                     & 54                                               & 83                                                                              & 36                                                                                    & 1.3x                                                                        & 0.6x                                                                          \\
				MsPacman                         & 92 $\pm$ 19                                     & 27                                               & 115                                                                             & 39                                                                                    & 2.8x                                                                        & 0.4x                                                                          \\
				Bowling                          & 77 $\pm$ 74                                     & 40                                               & 105                                                                             & 2865                                                                                  & 2.4x                                                                        & 0.7x                                                                          \\
				\bottomrule
			\end{tabular}
			}
		\end{center}
	\end{table}
    \clearpage
    
    \section{Dynamic SV-PPO Brax}
    \begin{table}[htbp]
    \caption{\textbf{Dynamic SV-PPO on Continuous Control} 
        \textbf{Column (a)} indicates the final converged return relative to PPO.
        \textbf{Column(b)} indicates the fraction of rounds $k$ where SV PPO updates the target policy. 
        \textbf{Columns (c) and (d)} show statistics about the convergence value ($\text{diff}_{k}/\overline{y}_{T}$ in Algorithm
		\ref{alg:convergence}), normalized by the threshold $\delta_{v}$, such that a value $\leq 1$ implies stability.
        \textbf{Column(e)} indicates value loss (MSBE) at convergence as a multiplier of PPO's. 
        \textbf{Column (f)} shows the size of the target policy jump as a multiplier of PPO's, demonstrating that SV-PPO executes substantially larger target updates. 
        \textbf{Column (g)} shows the same quantity for the behavioral policy.
        \textbf{Column (h)} shows mean policy entropy of PPO at convergence.}
    \label{tab:sv_brax_results_dynamic}
    \begin{center}
        \resizebox{\textwidth}{!}{
            \begin{tabular}{lcccccccc}
            \toprule
            Environment & \shortstack{(a)\\ \textbf{Score}\\ (\% of PPO)} & \shortstack{(b)\\ \textbf{Updates}\\(\% of PPO)} & \shortstack{(c)\\ {\boldmath \textbf{$\text{diff}_k /\delta_v$}} \\(Median)} & \shortstack{(d)\\ {\boldmath \textbf{Std. $\text{diff}_k$}} \\(\% of $\delta_v$)} & \shortstack{(e)\\ \textbf{Value Loss} \\ (x PPO)} & \shortstack{(f)\\ {\boldmath \textbf{KL($\pi_k || \pi_{k+1}$)}} \\(x PPO)} & \shortstack{(g)\\ {\boldmath \textbf{KL($\beta_k || \pi_{k+1}$)}} \\(x PPO)} & \shortstack{(h)\\ \boldmath$H(\pi_k)$ \\(PPO)} \\
            \midrule
            inverted\_pendulum & 137 $\pm$ 13 & 97.11 & 45 & 72 & 2.12x & 1.0x & 1.0x & -2.39 \\
            humanoidstandup & 110 $\pm$ 19 & 97.12 & 46 & 20 & 1.03x & 1.0x & 1.0x & 36.89 \\
            walker2d & 108 $\pm$ 31 & 97.23 & 108 & 37 & 0.93x & 1.0x & 1.0x & 1.82 \\
            inverted\_double\_pendulum & 103 $\pm$ 21 & 97.24 & 35 & 28 & 0.97x & 1.2x & 1.2x & -0.43 \\
            swimmer & 102 $\pm$ 3 & 96.91 & 46 & 24 & 1.08x & 1.2x & 1.1x & 2.53 \\
            halfcheetah & 100 $\pm$ 3 & 96.92 & 36 & 14 & 0.99x & 1.1x & 1.0x & 5.24 \\
            fast & 100 $\pm$ 0 & 97.22 & 9 & 6 & 0.99x & 1.1x & 1.0x & -0.23 \\
            pusher & 100 $\pm$ 16 & 97.03 & 63 & 46 & 0.83x & 1.3x & 1.2x & -14.62 \\
            hopper & 99 $\pm$ 4 & 97.20 & 51 & 24 & 0.80x & 1.1x & 1.1x & 3.95 \\
            ant & 99 $\pm$ 6 & 93.82 & 67 & 38 & 1.49x & 1.0x & 1.0x & -0.07 \\
            humanoid & 98 $\pm$ 2 & 97.18 & 48 & 21 & 1.08x & 1.0x & 1.0x & 28.95 \\
            reacher & 91 $\pm$ 21 & 97.20 & 83 & 37 & 0.89x & 1.0x & 1.0x & -7.16 \\
            \bottomrule
            \end{tabular}
        }
    \end{center}
\end{table}
    \clearpage
    \section{Static SV-PPO Brax}

\begin{table}[htbp]
\caption{\textbf{Static SV-PPO on Continuous Control.} 
        \textbf{Column (a)} indicates the final converged return relative to PPO.
        \textbf{Column(b)} indicates the fraction of rounds $k$ where SV PPO updates the target policy. 
        \textbf{Column(c)} indicates value loss (MSBE) at convergence as a multiplier of PPO's. 
        \textbf{Column (d)} shows the size of the target policy jump as a multiplier of PPO's, demonstrating that SV-PPO executes substantially larger target updates. 
        \textbf{Column (e)} shows the same quantity for the behavioral policy.
        \textbf{Column (f)} shows mean policy entropy of PPO at convergence.}
\label{tab:sv_brax_results_static}
\resizebox{\textwidth}{!}{\begin{tabular}{lcccccc}
\toprule
Environment & \shortstack{(a)\\Score \\ (\% of PPO)} & \shortstack{(b)\\Updates \\(\% of PPO)} & \shortstack{(c)\\ \textbf{Value Loss} \\(x PPO)} & \shortstack{(d)\\KL($\pi_k || \pi_{k+1}$)\\ (x PPO)} & \shortstack{(e)\\KL($\beta_k || \pi_{k+1}$)\\ (x PPO)} & \shortstack{(f)\\ \boldmath$H(\pi_k)$ \\(PPO)} \\
\midrule
inverted\_pendulum & 106 $\pm$ 49 & 12.49 & 2.17x & 5.0x & 0.5x & -2.39 \\
halfcheetah & 105 $\pm$ 4 & 12.49 & 0.84x & 7.9x & 1.3x & 5.24 \\
swimmer & 105 $\pm$ 15 & 12.49 & 0.90x & 12.7x & 1.1x & 2.53 \\
humanoidstandup & 104 $\pm$ 11 & 12.49 & 0.86x & 15.5x & 1.2x & 36.89 \\
fast & 100 $\pm$ 0 & 12.49 & 1.06x & 40.1x & 1.8x & -0.23 \\
humanoid & 98 $\pm$ 5 & 12.49 & 0.98x & 9.7x & 1.2x & 28.95 \\
inverted\_double\_pendulum & 97 $\pm$ 20 & 12.49 & 0.59x & 16.1x & 1.7x & -0.43 \\
ant & 93 $\pm$ 8 & 12.49 & 1.23x & 7.9x & 1.1x & -0.07 \\
reacher & 86 $\pm$ 11 & 12.49 & 0.84x & 11.7x & 1.0x & -7.16 \\
hopper & 84 $\pm$ 13 & 12.49 & 0.65x & 20.6x & 1.9x & 3.95 \\
pusher & 78 $\pm$ 16 & 12.49 & 0.91x & 72.0x & 7.0x & -14.62 \\
walker2d & 73 $\pm$ 19 & 12.49 & 1.01x & 17.8x & 1.2x & 1.82 \\
\bottomrule
\end{tabular}}
\end{table}

    \clearpage

    \section{Game-By-Game Comparison: Atari}
    \begin{figure}[htbp]
        \centering
        \includegraphics[width=\textwidth]{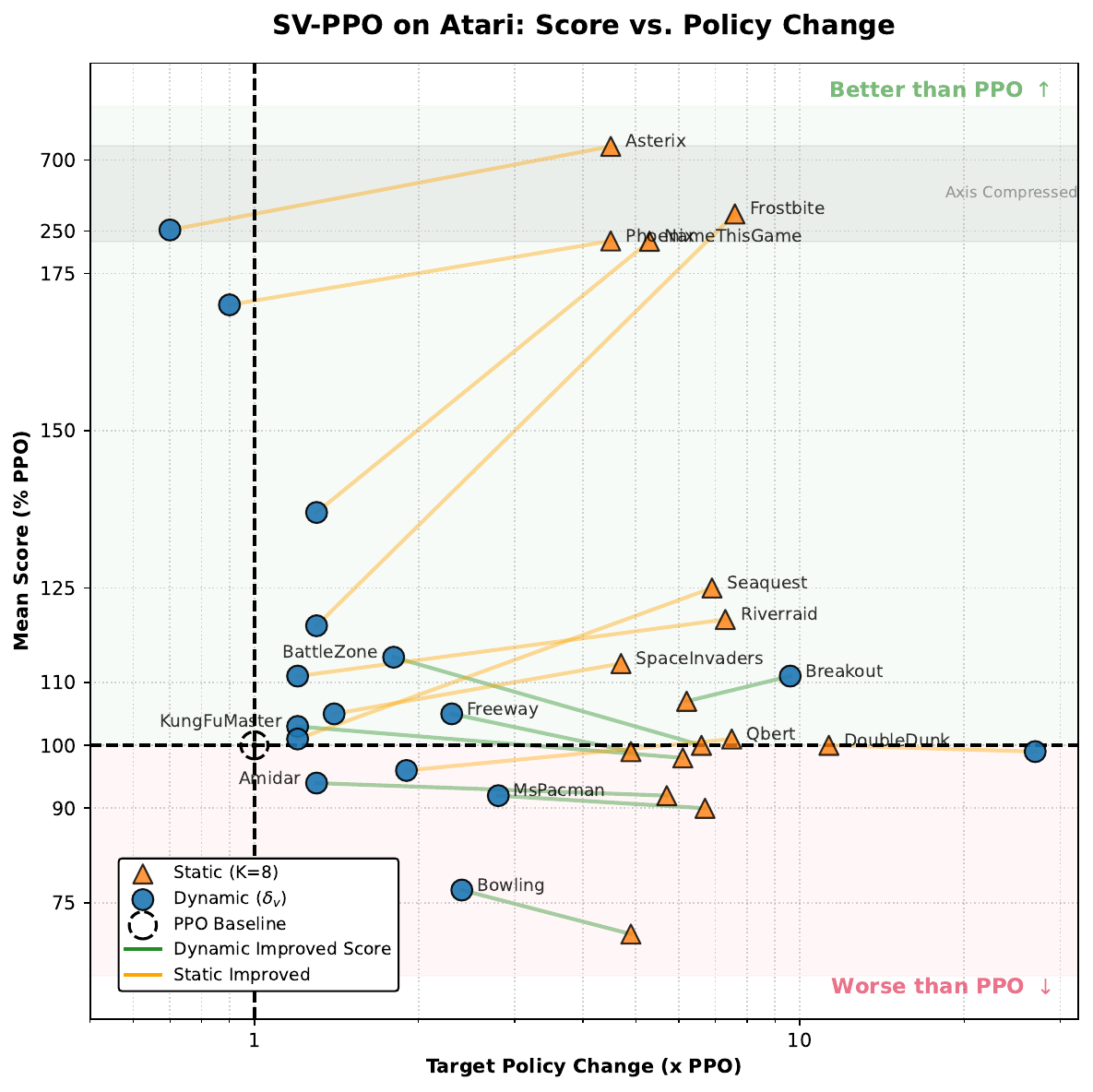}
        \caption{Game by Game comparison of Static and Dynamic SV-PPO on Atari. Each game is placed on the plot twice: a blue circle denotes the dynamic variant, and a triangle denotes the static variant. A colored line connects the same game for the two variants. The vertical axis is converged performance (normalized by PPO's) and the horizontal axis is the average change to the target policy, as measured by KL-divergence. Thus, if the line slopes ``upward'' for a certain game, this indicates the static variant increased performance. Moreover, the origin indicates PPO by construction.}
        \label{fig:comparison_atari}
    \end{figure}
    \clearpage
    
    \section{Game-By-Game Comparison: Brax}
    \begin{figure}[htbp]
        \centering
        \includegraphics[width=\textwidth]{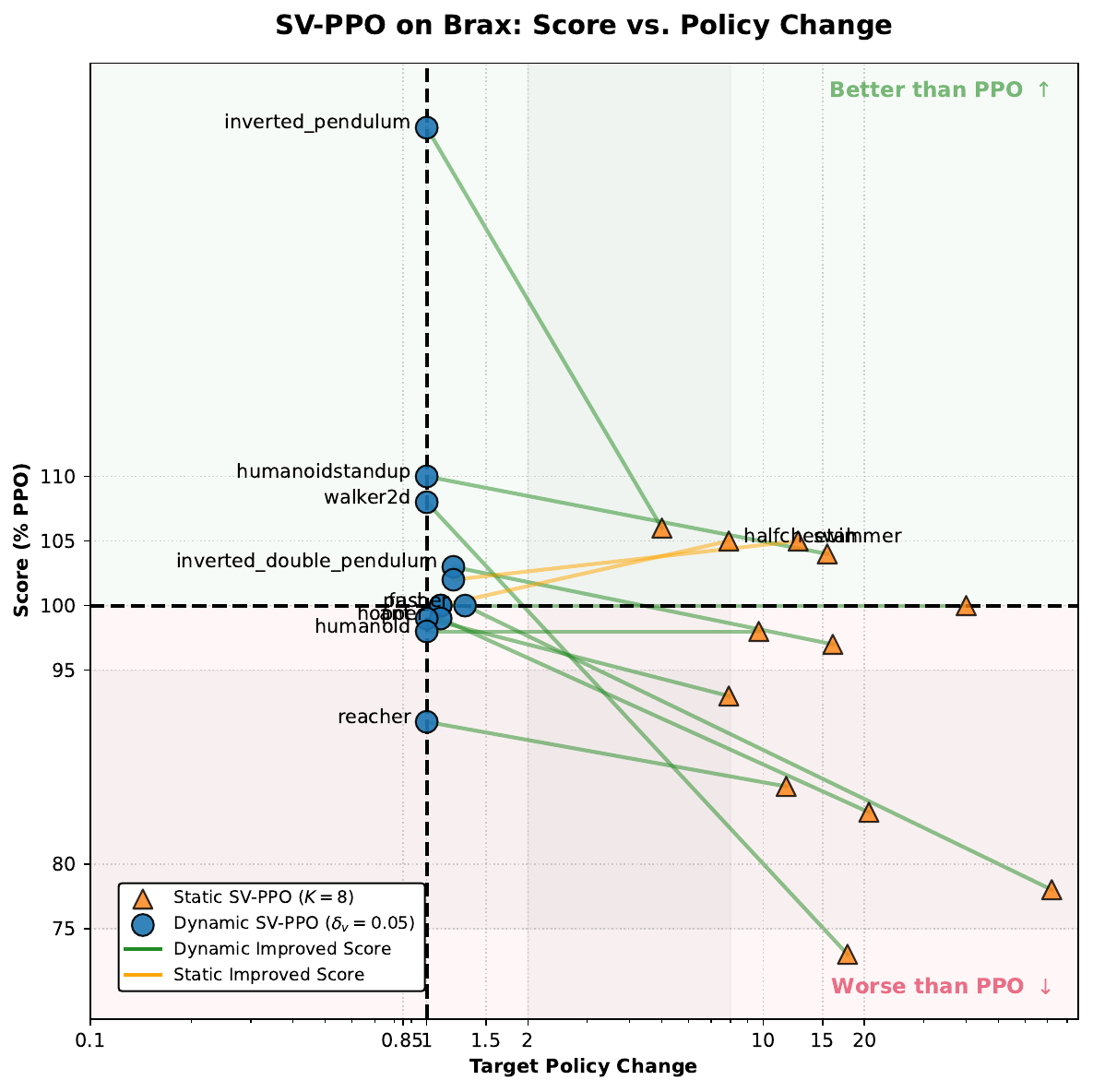}
        \caption{Game by Game comparison of Static and Dynamic SV-PPO on Brax. Each game is placed on the plot twice: a blue circle denotes the dynamic variant, and a triangle denotes the static variant. A colored line connects the same game for the two variants. The vertical axis is converged performance (normalized by PPO's) and the horizontal axis is the average change to the target policy, as measured by KL-divergence. Thus, if the line slopes ``upward'' for a certain game, this indicates the static variant increased performance. Moreover, the origin indicates PPO by construction.}
        \label{fig:comparison_brax}
    \end{figure}
    \clearpage
    \section{Learning Curves: Brax}
    
    \begin{figure}[htbp]
        \centering
        \includegraphics[width=\textwidth]{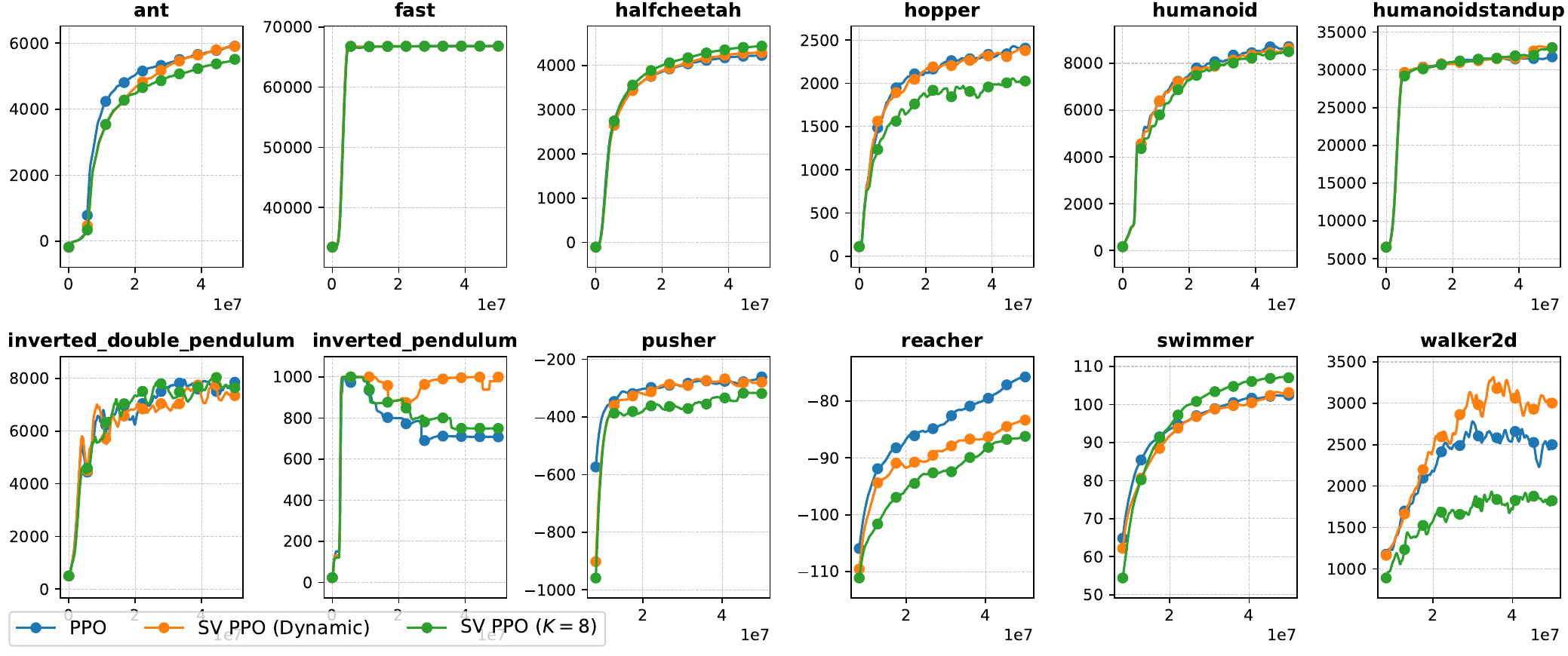}
        \caption{Learning Curves for PPO, Dynamic SV-PPO, and Static SV-PPO for Brax.}
        \label{fig:brax_sv_ppo_curves}
    \end{figure}
    \clearpage

    \section{SV-PPO Brax Hyperparameters}

    \begin{table}[htbp]
    \caption{\textbf{Hyperparameters for Brax Continuous Control.} Shared parameters (due to the PureJAXRL repository \cite{purejaxrl}) are applied across all baselines and variants. SV-PPO specific parameters apply only to the dynamic and static gated algorithms.}
    \label{tab:hyperparameters}
    \begin{center}
        \begin{tabular}{llc}
            \toprule
            \textbf{Category} & \textbf{Hyperparameter} & \textbf{Value} \\
            \midrule
            \multicolumn{3}{c}{\textbf{Shared PPO Hyperparameters}} \\
            \midrule
            \multirow{5}{*}{\textit{Environment \& Sampling}} 
            & Total Timesteps & $5 \times 10^7$ \\
            & Number of Environments & 2048 \\
            & Steps per Environment ($T$) & 10 \\
            & Discount Factor ($\gamma$) & 0.99 \\
            & GAE Parameter ($\lambda$) & 0.95 \\
            \midrule
            \multirow{10}{*}{\textit{Optimization (Network)}} 
            & Learning Rate (Initial) & $3 \times 10^{-4}$ \\
            & Learning Rate (Final) & $1 \times 10^{-5}$ \\
            & LR Schedule & Linear \\
            & Optimization Epochs & 4 \\
            & Minibatch Size & 1024 \\
            & Activation Function & Tanh \\
            & Max Gradient Norm & 1.0 \\
            & Initial Log Std & 0.5 \\
            & Normalize Observations & True \\
            & Normalize Rewards & True \\
            & Warmup Steps & 1000 \\
            \midrule
            \multirow{4}{*}{\textit{Loss Coefficients}} 
            & PPO Clip Range ($\epsilon$) & 0.2 \\
            & Value Clip Range & 0.2 \\
            & Value Coefficient ($c_2$) & 0.5 \\
            & Entropy Coefficient ($\alpha_{ent}$) & 0.001 \\
            
            \midrule
            \multicolumn{3}{c}{\textbf{Dynamic SV-PPO Specific Hyperparameters}} \\
            \midrule
            \multirow{6}{*}{\textit{Stability Gating}} 
            & Value Diff Threshold ($\delta_v$) & 0.05 \\
            & Min Stable Iters ($K_{min}$) & 2 \\
            & Max Policy Delay ($K_{max}$) & 8 \\
            & IS Ratio Bounds ($\bar{c}, \bar{\rho}$) & $[0.0, 100.0]$ \\
            & $K$ Decay Fraction & 0.05 \\
            \midrule
            \multicolumn{3}{c}{\textbf{Static SV-PPO Specific Hyperparameters}} \\
            \midrule
            & Min Stable Iters ($K_{min}$) & 8 \\
            & Max Policy Delay ($K_{max}$) & 8 \\
            & IS Ratio Bounds ($\bar{c}, \bar{\rho}$) & $[0.0, 100.0]$ \\
            \bottomrule
        \end{tabular}
    \end{center}
\end{table}
\clearpage

\section{SV-PPO Atari Implementation and Hyperparameters}
\begin{table}[htbp]
    \caption{\textbf{Hyperparameters for Atari-16.} The PPO implementation is due to PureJAXRL, with most hyperparamters due to the original paper \cite{schulman2017proximalpolicyoptimizationalgorithms}. However, following Museli's \cite{museli} strong implementation of PPO \textit{with sticky actions}, we made several changes. First, separate IMPALA networks were used for actor and critic.  The discount factor was raised to $\gamma = 0.995$ and the maximum learning rate to $3 \times 10^{-4}$. The minibatch size was also increased to 1,024. These changes improved performance on all tested games, and were shared by both runs. The SV PPO stability gating hyperparamters were tuned on the MinAtar environments \cite{young19minatar}, which can be run much more quickly than experiments on Atari. SV-PPO specific parameters apply only to the dynamic and static gated algorithms.}
    \label{tab:atari_hyperparameters}
    \begin{center}
        \begin{tabular}{llc}
            \toprule
            \textbf{Category} & \textbf{Hyperparameter} & \textbf{Value} \\
            \midrule
            \multicolumn{3}{c}{\textbf{Shared PPO Hyperparameters}} \\
            \midrule
            \multirow{10}{*}{\textit{Environment \& Sampling}} 
            & Total Timesteps & $200M$ \\
            & Number of Environments & 64 \\
            & Steps per Environment ($T$) & 128 \\
            & Discount Factor ($\gamma$) & 0.995 \\
            & GAE Parameter ($\lambda$) & 0.95 \\
            & Frame Skip & 4 \\
            & No-op Max & 30 \\
            & Repeat Action Probability & 0.25 \\
            & Episodic Life & True \\
            & Reward Clipping & False \\
            \midrule
            \multirow{7}{*}{\textit{Optimization (Network)}} 
            & Learning Rate (Initial) & $3 \times 10^{-4}$ \\
            & Learning Rate (Final) & $1 \times 10^{-6}$ \\
            & LR Schedule & Linear \\
            & Optimization Epochs & 4 \\
            & Minibatch Size & 1024 \\
            & Max Gradient Norm & 1.0 \\
            & Layer Normalization & True \\
            \midrule
            \multirow{4}{*}{\textit{Loss Coefficients}} 
            & PPO Clip Range ($\epsilon$) & 0.1 \\
            & Value Clip Range & 0.2 \\
            & Value Coefficient ($c_2$) & 0.25 \\
            & Entropy Coefficient ($\alpha_{ent}$) & 0.01 \\
            
            \midrule
            \multicolumn{3}{c}{\textbf{Dynamic SV-PPO Specific Hyperparameters}} \\
            \midrule
            \multirow{5}{*}{\textit{Stability Gating}} 
            & Value Diff Threshold ($\delta_v$) & 0.01 \\
            & $K_{min}$ Decay Fraction & 0.2 \\
            & Min Stable Iters ($K_{min}$) & 9 decays to 1 \\
            & Max Policy Delay ($K_{max}$) & 33 \\
            & IS Ratio Bounds ($\bar{c}, \bar{\rho}$) & $[0.0, 5.0]$ \\
            \midrule
            \multicolumn{3}{c}{\textbf{Static SV-PPO Specific Hyperparameters}} \\
            \midrule
            \multirow{3}{*}{\textit{Stability Gating}}
            & Min Stable Iters ($K_{min}$) & 9 \\
            & Max Policy Delay ($K_{max}$) & 9 \\
            & IS Ratio Bounds ($\bar{c}, \bar{\rho}$) & $[0.0, 5.0]$ \\
            \bottomrule
        \end{tabular}
    \end{center}
\end{table}
\newpage

\section{SV-PPO}
    \label{appendix:sv_ppo}
	Let $\pi_{\theta}$ denote the policy actor network, with $\theta_{k}$ parameterizing
	the current behavioral policy at round $k$, and $\overline{\theta}_{k}$ denoting the target policy
	parameters. The PPO policy optimization objective stays the same, except that the ratio $r_{k}(\theta)$
	is the ratio of behavioral policies.

	\begin{align}
		J^{CLIP}(\theta) = \hat{\mathbb{E}}_{t}\left[ \min \left( \frac{\beta_{\theta}(a_{t}| s_{t})}{\beta_{\theta_{k}}(a_{t}| s_{t})}\hat{A}_{t}, \text{clip}\left(\frac{\beta_{\theta}(a_{t}| s_{t})}{\beta_{\theta_{k}}(a_{t}| s_{t})}, 1 - \epsilon, 1 + \epsilon\right) \hat{A}_{t}\right) \right]
	\end{align}
	In this way, the behavioral policy is allowed to drift further and further away from the target policy in
	each round. Once the convergence criteria is met, the update is just $\overline{\theta}\gets \theta_{k+1}$.

	The critic training process must ensure that the value function's fixed point corresponds to $\pi_{\overline{\theta}_k}$.
	Let the value network be $v_{\phi}$. We use the clipped importance sampling method V-Trace \cite{espeholt2018impalascalabledistributeddeeprl},
	which can be written in recursive terms:
	\begin{align}
		\rho_{t}   & = \min \left( \bar{\rho}, \frac{\pi_{\overline{\theta}_k}(a_{t}| x_{t})}{\pi_{\theta_k}(a_{t}| x_{t})}\right), \quad c_{i}= \min \left( \lambda, \frac{\pi_{\overline{\theta}_k}(a_{i}| x_{i})}{\pi_{\theta_k}(a_{i}| x_{i})}\right) \label{eq:is_weights} \\
		\delta_{t} & = r_{t}+ \gamma v_{\phi}(s_{t+1}) - v_{\phi}(s_{t}) \nonumber                                                                                                                                                                                              \\
		y_{t}      & = v_{\phi}(s_{t}) + \rho_{t}\delta_{t}+ \gamma c_{t}(y_{t+1}- v_{\phi}(s_{t+1})) \nonumber                                                                                                                                                                 \\
		           & \text{with }\quad y_{T+1}= v_{\phi}(s_{T+1}) \label{eq:vtrace}
	\end{align}
	Since $c_{t}$ is at most $\lambda$, this results in TD($\lambda$) when on-policy ($\theta_{k}$ =
	$\overline{\theta}_{k}$). To eliminate bias due to off policy sampling, $\overline{\rho}$ must be set to
	infinity. Clipping with $\overline{\rho}< \infty$ stabilizes learning at the cost of adding some bias. We
	set $\overline{\rho}$ to 5. Since the advantage estimates are also off-policy, we modify the
	generalized advantage estimation \citep{gae} used by PPO by incorporating the importance sampling weights
	from Equation \eqref{eq:is_weights}, which are originally due to Retrace($\lambda$) \citep{DBLP:retrace}.
	\begin{align}
		\hat{A}_{t} & = \delta_{t}+ \gamma c_{t}\hat{A}_{t+1}, \quad \hat{A}_{T+1}= 0 \label{eq:retrace_gae}
	\end{align}
	Finally, we normalize the estimates $\hat{A}_{t}$ by dividing by the batch standard deviation, but we
	do not center them, to avoid flipping the sign.

\subsubsection*{Acknowledgments}
\label{sec:ack}
This material is based upon work supported by ARO grant \#W911NF2210251. Any opinions, findings, and conclusions or recommendations expressed in this material are those of the authors and do not necessarily reflect the views of the ARO.
   
\end{document}